\documentclass{article}

\usepackage{microtype}
\usepackage{graphicx}
\usepackage{algorithmic}
\usepackage{booktabs} 
\usepackage{amsmath}
\usepackage{amsthm}
\usepackage{amssymb} 
\usepackage{color, colortbl}
\usepackage{xcolor}
\usepackage{hyperref}
\usepackage[capitalise]{cleveref}

\usepackage{enumitem}
\usepackage{bm}
\usepackage{bbm}
\usepackage{graphicx}
\usepackage{thmtools}
\usepackage{thm-restate}
\usepackage{mathtools}
\usepackage[skip=0pt]{caption}
\usepackage[skip=0pt]{subcaption}
\usepackage{booktabs} 
\usepackage{fontawesome}
\usepackage{enumitem}
\usepackage{accents}

\usepackage{threeparttable, array, float} 

\definecolor{red}{HTML}{E51400}  
\definecolor{blue}{HTML}{0050EF} 
\definecolor{green}{HTML}{008A00} 
\definecolor{purple}{HTML}{AA00FF} 
\definecolor{dark-red}{rgb}{0.4, 0.15, 0.15}
\definecolor{dark-blue}{rgb}{0.15, 0.15, 0.4}
\definecolor{medium-red}{rgb}{0.5, 0, 0}
\definecolor{medium-blue}{rgb}{0, 0, 0.5}
\definecolor{light-red}{rgb}{0.7, 0, 0}
\definecolor{light-blue}{rgb}{0, 0, 0.7}

\hypersetup{
    colorlinks=true,
    linkcolor=blue,
    filecolor=magenta,      
    urlcolor=blue,
}

\newtheorem{theorem}{\bf Theorem}
\newtheorem{lemma}{\bf Lemma}

\newtheorem{condition}{\bf Condition}

\newtheorem{definition}{\bf Definition}

\theoremstyle{definition}
\newtheorem{remark}{\bf Remark}

\definecolor{red}{HTML}{E51400} 
\definecolor{blue}{HTML}{0050EF} 
\definecolor{green}{HTML}{008A00} 
\definecolor{purple}{HTML}{AA00FF} 
\definecolor{orange}{HTML}{FF7F00}
\definecolor{gray}{HTML}{848482}
\definecolor{Gray}{gray}{0.85}
\definecolor{LightGray}{gray}{0.96}

\usepackage{constants}
\newcommand{\constlabel}[1]{\pageref*{sec:unknown_trans},\arabic{#1}}
\newconstantfamily{s}{
format=\constlabel,
reset={subsection}
}

\DeclareMathOperator*{\argmax}{argmax}

\newcommand{\ubar}[1]{\underaccent{\bar}{#1}}

\newcommand{\norm}[1]{\left\lVert#1\right\rVert}

\newcommand{\cS}{\mathcal{S}}
\newcommand{\abs}[1]{\left| #1 \right|}
\newcommand{\uabs}[1]{\left[ #1 \right]_+}

\newcommand{\MTPMr}{{MTPM+}}

\newcommand{\R}{\mathbb{R}}

\newcommand{\E}{\mathbb{E}}

\newcommand{\Var}{{\rm Var}}

\newcommand{\I}{\mathbb{I}}

\newcommand{\bc}{\boldsymbol{c}}
\newcommand{\be}{\boldsymbol{e}}

\newcommand{\bp}{\boldsymbol{p}}

\newcommand{\bu}{\boldsymbol{u}}
\newcommand{\bV}{\boldsymbol{V}}
\newcommand{\bv}{\boldsymbol{v}}
\newcommand{\bw}{\boldsymbol{w}}

\newcommand{\bX}{\boldsymbol{X}}

\newcommand{\bx}{\boldsymbol{x}}
\newcommand{\by}{\boldsymbol{y}}

\newcommand{\bmu}{\boldsymbol{\mu}}

\newcommand{\bDelta}{\boldsymbol{\Delta}}

\newcommand{\cA}{\mathcal{A}}

\newcommand{\cC}{\mathcal{C}}
\newcommand{\cD}{\mathcal{D}}
\newcommand{\cE}{\mathcal{E}}
\newcommand{\cF}{\mathcal{F}}

\newcommand{\cR}{\mathcal{R}}

\newcommand{\cO}{{\mathcal{O}}}
\newcommand{\cP}{{\mathcal{P}}}

\newcommand{\cZ}{{\mathcal{Z}}}

\newcommand{\logL}{\log\left(\frac{SAHT}{\delta'}\right)}

\newcommand{\tDelta}{\Delta}

\newcommand{\reg}{\text{Reg}}
\newcommand{\df}{\text{def}}

\makeatletter
\newcommand*{\rom}[1]{\expandafter\@slowromancap\romannumeral #1@}
\makeatother

\newcommand{\ts}[1]{}

\newcommand{\compilefullversion}{true}
\ifthenelse{\equal{\compilefullversion}{false}}{%
	\newcommand{\OnlyInFull}[1]{}
	\newcommand{\OnlyInShort}[1]{#1}
}{%
	\newcommand{\OnlyInFull}[1]{#1}%
	\newcommand{\OnlyInShort}[1]{}%
}%

\newcommand{\compilehidecomments}{false}
\ifthenelse{ \equal{\compilehidecomments}{true} }{%
	\newcommand{\wei}[1]{}
	\newcommand{\xutong}[1]{}
	\newcommand{\jinhang}[1]{}
	\newcommand{\siwei}[1]{}
        \newcommand{\carlee}[1]{}

}{

\newcommand{\wei}[1]{{\color{blue}{[Wei: #1]}}}
\newcommand{\xutong}[1]{{\color{green} [Xutong: #1]}}
\newcommand{\jinhang}[1]{{\color{orange} [\text{Jinhang:} #1]}}
\newcommand{\siwei}[1]{{\color{red} [\text{Siwei:} #1]}}

}

\allowdisplaybreaks

\newenvironment{talign*}
 {\csname align*\endcsname}
 {\endalign}


\renewcommand{\algorithmiccomment}[1]{\hfill$\triangleright$\textit{\textcolor{blue}{#1}}}

\usepackage[accepted]{icml2024}

\icmltitlerunning{Combinatorial Multivariant Multi-Armed Bandits}

\begin{document}

\twocolumn[
\icmltitle{Combinatorial Multivariant Multi-Armed Bandits \\with Applications to Episodic Reinforcement Learning and Beyond}





\begin{icmlauthorlist}
\icmlauthor{Xutong Liu}{cuhk,umass}
\icmlauthor{Siwei Wang}{msra}
\icmlauthor{Jinhang Zuo}{umass,caltech}
\icmlauthor{Han Zhong}{pku}
\icmlauthor{Xuchuang Wang}{umass}
\icmlauthor{Zhiyong Wang}{cuhk}
\icmlauthor{Shuai Li}{sjtu}
\icmlauthor{Mohammad Hajiesmaili}{umass}
\icmlauthor{John C.S. Lui}{cuhk}
\icmlauthor{Wei Chen}{msra}
\end{icmlauthorlist}

\icmlaffiliation{cuhk}{The Chinese University of Hong Kong, Hong Kong SAR, China}
\icmlaffiliation{umass}{University of Massachusetts Amherst,  Massachusetts, United States}
\icmlaffiliation{caltech}{California Institute of Technology, California, United States}
\icmlaffiliation{pku}{Peking University, Beijing, China}
\icmlaffiliation{msra}{Microsoft Research, Beijing, China}
\icmlaffiliation{sjtu}{Shanghai Jiao Tong University, Shanghai, China}
\icmlcorrespondingauthor{Shuai Li}{shuaili8@sjtu.edu.cn}
\icmlcorrespondingauthor{Siwei Wang}{siweiwang@microsoft.com}
\icmlcorrespondingauthor{Wei Chen}{weic@microsoft.com}
\icmlkeywords{Machine Learning, ICML}

\vskip 0.3in
]



\printAffiliationsAndNotice{}  

\begin{abstract}

We introduce a novel framework of combinatorial multi-armed bandits (CMAB) with multivariant and probabilistically triggering arms (CMAB-MT), where the outcome of each arm is a $d$-dimensional multivariant random variable and the feedback follows a general arm triggering process. Compared with existing CMAB works, CMAB-MT not only enhances the modeling power but also allows improved results by leveraging distinct statistical properties for multivariant random variables. For CMAB-MT, we propose a general 1-norm multivariant and triggering probability-modulated smoothness condition, and an optimistic CUCB-MT algorithm built upon this condition. Our framework can include many important problems as applications, such as episodic reinforcement learning (RL) and probabilistic maximum coverage for goods distribution, all of which meet the above smoothness condition and achieve matching or improved regret bounds compared to existing works. Through our new framework, we build the first connection between the episodic RL and CMAB literature, by offering a new angle to solve the episodic RL through the lens of CMAB, which may encourage more interactions between these two important directions.

\end{abstract}

\section{Introduction}
The stochastic multi-armed bandit (MAB)~\cite{robbins1952some,auer2002finite} is a classical model for sequential decision-making that has been widely studied (cf.~\citet{slivkins2019introduction,lattimore2020bandit}). As a noteworthy extension of MAB, combinatorial multi-armed bandits (CMAB) have drawn considerable attention due to their rich applications in domains such as online advertising, network optimization, and healthcare systems \cite{gai2012combinatorial, kveton2015combinatorial, chen2013combinatorial, chen2016combinatorial,wang2017improving,merlis2019batch,liu2021multi,zuo2021combinatorial,zuo2022online}.
In CMAB, the learning agent chooses a combinatorial action (often referred to as super arm) in each round. This combinatorial action would trigger a set of arms to be pulled simultaneously, and the outcomes of these arms are observed as feedback (typically known as semi-bandit feedback). The agent then receives a reward, which can be a general function of the pulled arms' outcomes, with the summation function being the most common example.
The agent's goal is to minimize the expected \textit{regret}, which quantifies the difference in expected cumulative rewards between always selecting the best action (i.e., the action with the highest expected reward) and following the agent's own policy. CMAB poses the challenge of balancing exploration and exploitation while dealing with an exponential number of combinatorial actions.

To model a wider range of application scenarios where the combinatorial action may probabilistically trigger arms, 
\citet{chen2016combinatorial} first introduce a generalization of CMAB, known as CMAB with probabilistically triggered arms (CMAB-T). This extension successfully encompasses a broader range of applications, including cascading bandits~\cite{combes2015combinatorial} and online influence maximization (OIM)~\cite{wen2017online}. Subsequently, \citet{wang2017improving,liu2022batch} improve the regret bounds of \cite{chen2016combinatorial} by introducing novel triggering probability modulated (TPM) smoothness conditions and/or variance adaptive algorithms. Further elaboration on related works can be found in \cref{apdx_sec:related_work}.

Despite the expanded modeling capabilities and improved regret bounds, all prior CMAB-T frameworks assume that each arm's outcome is a univariate random variable, upon which they base their smoothness conditions, algorithms, and analyses. In real-world applications, arm outcomes can be $d$-dimensional multivariate random variables with distinct statistical properties. One example is the indivisible goods distribution \cite{alkan1991fair,chevaleyre2017distributed}, where each good can be distributed to one of $d$ target users, forming a multivariant random variable.
Another critical example is episodic reinforcement learning (RL) \cite{jaksch2010near,azar2017minimax,zanette2019tighter,neu2020unifying},  
where in each episode the agent starts from an initial state $s_1$ and transits through traverses a series of $H$ states $(s_h)_{h\in[H]}$ by taking action $a_h$ upon each encountered state $s_h$. Each transition in this scenario is a multivariant random variable, with outcomes spanning the state space $S$.
Existing CMAB-T approaches cannot effectively model these situations with appropriate smoothness conditions. Specifically, they resort to treating each multivariate arm as $d$ separate arms, thereby neglecting the unique statistical characteristics of multivariate random variables and yielding suboptimal regret performance.

\textbf{Our Contributions.} 
We introduce a new CMAB-MT framework, which inherits the arm triggering mechanism of CMAB-T while accommodating $d$-dimensional multivariate random variables as arm outcomes. The key challenge lies in determining the contributions of each dimension of the arms to the overall regret and effectively leveraging the multivariate statistical characteristics. To address this challenge, we first introduce a novel 1-norm MTPM smoothness condition that assigns varying weights to different arms and dimensions, which flexibly covers existing 1-norm TPM smoothness conditions of CMAB-T and accommodates new applications such as episodic RL. Then, we construct an \textit{action-dependent} confidence region that can incorporate problem-specific multivariant statistical properties. Leveraging this confidence region, we devise the CUCB-MT algorithm with a general joint oracle and establish the first regret bound for any CMAB-MT problem. Our new analysis combines regret decomposition techniques from the RL domain and sharp CMAB techniques to deal with arm triggering and regret amortization, which can yield matching or improved results for applications within and beyond this study.  

To show the applicability of our framework, we first show that episodic RL fits into the CMAB-MT framework by mapping each transition kernel as an arm and employing the occupancy measure as the triggering probability. Leveraging this insight, we give two CUCB-MT algorithms that can achieve $\tilde{O}(\sqrt{H^4S^2AT})$ and $\tilde{O}(\sqrt{H^3SAT})$ regret based on distinct smoothness conditions, with the latter matches the lower bound \cite{jaksch2010near} up to logarithmic factors. Remarkably, our regret bound improves at least a factor of $O(\log ^{1/2}T)$ for the leading regret term compared with existing works \cite{zanette2019tighter,zhang2021isreinforcement,zhang2023settling} owing to our sharp CMAB analysis. As a by-product, our framework gives a gap-dependent regret that scales with $O(\log T)$.
Notably, episodic RL is widely known to be a strict generalization of the MAB and thus much harder to solve than MAB due to the state transition and long-term reward structure.
Our work makes the first attempt to view episodic RL as an instance of CMAB, and offers a new angle for addressing episodic RL challenges through the lens of CMAB. Our results highlight that episodic RL is not significantly harder than CMAB-MT problems, and
build a valuable connection between the RL and CMAB that may encourage more interactions between these pivotal research directions.

Furthermore, we explore another application beyond episodic RL that fits into our framework: the probabilistic maximum coverage for goods distribution (PMC-GD). For PMC-GD, we overcome the challenge of identifying a tight confidence region based on its unique statistical property and finding the efficient implementation of the joint oracle. To this end, our framework gives a regret bound that improves the best-known variance-adaptive algorithm~\cite{merlis2019batch,liu2022batch} by a factor of $\tilde{O}(\sqrt{|V|/k})$, where $|V|$ and $k$ are the numbers of target nodes and selected source nodes, and $V\gg k$ in most application scenarios \cite{chen2016combinatorial,liu2023variance}. 

\section{Combinatorial MAB with Multivariant and  Probabilistically Triggering Arms}
In this section, we present the combinatorial multi-armed bandit with multivariant and probabilistically triggering arms (or CMAB-MT for short), which generalizes the previous CMAB-T framework to handle $d$-dimensional multivariant arm outcomes. CMAB-MT covers not only existing instances of CMAB-T with univariant arms, but more importantly, the episodic RL as a new example.

\textbf{Notations.} 
We use \textbf{boldface} symbols for vectors $\bv\in \R^d$.
For matrix $\bv\in \R^{m\times d}$, we treat $\bv$ as a long column vector that sequentially stacks $m$ sub-vectors of dimension $d$ and $\bv_i\in\R^d$ is the $i$-th sub-vector for $i\in[m]$. For function $V:[d]\rightarrow \R$, we use $\boldsymbol{V}$ to denote the vector $(V(x))_{x\in [d]}\in \R^d$. For any set $S$, we define probability simplex $\bDelta_S=\{\bp\in [0,1]^{|S|}: \sum_{i\in S}p(i)=1\}$. We use $\be_i \in \R^d$ to denote the vector whose $i$-th entry is $1$ and $0$ elsewhere. For vector $\bv \in \R^d$, we use $|\bv|$ to denote the vector $(|v_i|)_{i\in[d]}$.

\subsection{Framework Setup}
\textbf{Problem Instance.} A CMAB-MT problem instance can be described by a tuple $([m], d, \Pi, \cD, D_{\text{trig}},R)$,
	where $[m]=\{1,2,...,m\}$ is the set of multivariant base arms; 
 $d$ is the dimension of multivariant base arm's random outcome\footnote{For simplicity, we assume dimensions are the same, yet it is easy to generalize $d$ to $d_i$ for arm $i \in [m]$.} (with bounded support $[0,1]^d$), i.e., outcome $\bX_i=(X_{i,1}, ..., X_{i,d}) \in [0,1]^d$;
	$\Pi$ is the set of eligible combinatorial actions and $\pi \in \Pi$ is a combinatorial action;\footnote{When $\Pi$ is a collection of subsets of $[m]$, we
	call action $\pi\in \Pi$ a super arm. Otherwise, we treat $\Pi$ as a general action space, same as in \cite{wang2017improving}.}
	$\cD$ is the set of possible distributions over the outcomes of base arms with support $[0,1]^{m\times d}$;
$ D_{\text{trig}}$ is the probabilistic triggering function and $R$ is the reward function, which shall be specified shortly after.

\textbf{Learning Process.} In CMAB-MT, the learning agent interacts with the unknown environment in a sequential manner as follows.
First, the environment chooses a distribution $D \in \cD$ unknown to the agent.
Then, at round $t=1,2,...,T$, the agent selects a combinatorial action $\pi_t \in \Pi$ and the environment draws from the unknown distribution $D$ random outcome vectors $\bX_t=(\bX_{t,1},...,\bX_{t,m})\in [0,1]^{m\times d}$ for all $m$ multivariate base arms.
Note that the outcome $\bX_t$ is assumed to be independent from outcomes generated in previous rounds, but outcomes $\bX_{t,i}$ and $\bX_{t,j}$ in the same round could be correlated.
Let $D_{\text{trig}}(\pi,\bX)$ be a distribution over all possible subsets of $[m]$.
When the action $\pi_t$ is played on the outcome $\bX_t$, base arms in a random set $\tau_t \sim D_{\text{trig}}(\pi_t, \bX_t)$ are triggered, 
	meaning that the multivariant outcomes of arms in $\tau_t$, i.e., $(\bX_{t,i})_{i\in \tau_t}$ are revealed as the feedback to the agent, and are involved in determining the reward of action $\pi_t$. 
Function $D_{\text{trig}}$ is referred to as the \textit{probabilistic triggering function}.
At the end of the round $t$, the agent will receive a non-negative reward $R(\pi_t, \bX_t, \tau_t)$, determined by $\pi_t,
\bX_t$ and $\tau_t$.

\textbf{Learning Objective.} 
The goal of CMAB-MT is to accumulate as much reward as possible over $T$ rounds, by learning distribution $D$ or its parameters.
Let vector $\bmu=(\bmu_1,...,\bmu_m)\in [0,1]^{m \times d}$, where $\bmu_i=(\mu_{i,1}, ..., \mu_{i,d})\in [0,1]^d$ denote the mean vector of base arm $i$'s multivariant outcome, i.e., $\bmu_i=\E_{\bX_{t}\sim D}[\bX_{t,i}]$.
Similar to the CMAB-T~framework~\cite{wang2017improving,liu2022batch,liu2023contextual}, we assume that the expected reward $\E[R(\pi,\bX,\tau)]$ is a function of 
the unknown mean vector $\bmu$, where the expectation is taken over the randomness of $\bX\sim D$ and $\tau \sim D_{\text{trig}}(\pi,\bX)$, and therefore we use $r(\pi;\bmu)\overset{\df}{=}\E[R(\pi,\bX,\tau)]$ to denote the expected reward. 
To allow the algorithm to estimate the mean $\bmu_i$ directly from samples, we assume the outcome does not depend on whether the arm $i$ is triggered, i.e., $\E_{\bX \sim D, \tau \sim D_{\text{trig}}(S,\bX)}[\bX_i | i\in \tau]=\E_{\bX\sim D}[\bX_i]$.
The performance of a learning algorithm ALG
is measured by its {\em regret}, defined as the difference of the expected cumulative reward between always playing the best action $\pi^* \overset{\df}{=} \argmax_{\pi \in \Pi}r(\pi;\bmu)$ and that of playing actions chosen by the algorithm.

For many reward functions $r(\pi;\bmu)$, it is NP-hard to compute the exact $\pi^*$ even when $\bmu$ is known~\cite{chen2013combinatorial,wang2017improving,liu2022batch}, so we assume that one has access to an $(\alpha, \beta)$-approximation oracle $\tilde{\cO}$. $\tilde{\cO}$ takes a confidence region function $\cC$ that maps any action $\pi\in \Pi$ to possible parameters $\cC(\pi)\subseteq [0,1]^{m\times d}$ as input, and outputs an action-parameter pair $(\tilde{\pi},\tilde{\bmu})=\tilde{\cO}(\cC)$ such that $\tilde{\pi}\in \Pi$, $\tilde{\bmu}\in \cC(\tilde{\pi})$ and $(\tilde{\pi},\tilde{\bmu})$ is an $\alpha$-approximation with probability at least $\beta$, i.e., $\Pr\left[r(\tilde{\pi},\tilde{\bmu})\ge \alpha \cdot \max_{\pi\in \Pi, \bmu\in \cC(\pi)}r(\pi;\bmu)\right] \ge \beta$.
Formally, the $T$-round $(\alpha, \beta)$-approximate regret is defined as
\begin{equation}
    \text{Reg}(T;\alpha, \beta, \bmu)= T \cdot \alpha\beta \cdot r(\pi^*;\bmu)-\E\left[\sum_{t=1}^Tr(\pi_t;\bmu)\right],
\end{equation}
where the expectation is taken over the randomness of the outcomes $\bX_1, ..., \bX_T$, the triggered sets $\tau_1, ..., \tau_T$, as well as the randomness of the algorithm ALG itself.

\begin{remark}[CMAB-MT v.s. CMAB-T]
    CMAB-MT is more general and reduces to CMAB-T when $d=1$. Conversely, for any CMAB-MT instance, one can treat each multi-variant arm $i\in[m]$ as $d$ separate arms $(i,j)_{j\in [d]}$ with unknown mean $\mu_{i,j}$ and use the CMAB-T model to learn these $md$ arms, but in this way, one cannot enjoy some nice statistical property (e.g., concentration properties) by treating them as a whole. Take  PMC-GD in \cref{sec:CDP} as an example, using CMAB-MT can improve the regret up to a factor of $O(\sqrt{d})$, owing to tighter concentration inequality around the mean vector $\bmu_i$ instead of using $d$ separate concentration inequality around $\mu_{i,j}$. 
\end{remark}

\begin{remark}[Computational complexity of joint oracle $\tilde{\cO}$]
    Different from the classical CMAB $(\alpha,\beta)$-approximation oracle $\cO:[0,1]^{m\times d}\rightarrow \Pi$ that takes a single parameter (e.g., mean vector $\bmu$) as input and optimizes over the action space $\Pi$, the joint oracle $\tilde{\cO}$ can optimize over the joint action-parameter space $(\pi, \cC(\pi))_{\pi\in\Pi}$. 
    In the worst case, one has to compute the best reward $r^*(\pi)=\max_{\bmu\in \cC(\pi)}r(\pi;\bmu)$ for $\pi\in \Pi$, and then enumerate over $\pi\in \Pi$ to get the optimal $(\pi^*,\bmu^*)$ in time $O(|\Pi|)$. Nevertheless, the joint oracle has been used in many CMAB \cite{combes2015combinatorial,degenne2016combinatorial} and linear contextual bandits (Sec. 19.3.1 in  \citet{lattimore2020bandit}) works. In this paper, we will show that all CMAB-MT applications considered, i.e., the episodic RL (\cref{sec:RL}) and PMC-GD (\cref{sec:CDP}) have efficient implementations for the joint oracle. 
\end{remark}

\subsection{Key Quantities and Conditions} 
In the CMAB-MT model, there are several quantities and conditions that are crucial to the subsequent study.
First, we define \textit{{triggering probability}} $q_i^{D, D_{\text{trig}}, \pi}$ as the probability that base arm $i$ is triggered when the combinatorial action is $\pi$, the outcome distribution is $D$, and the probabilistic triggering function is $D_{\text{trig}}$.
Since $D_{\text{trig}}$ is always fixed in a given application context, and $D$ often determines the triggering probability via its mean $\bmu$ in most cases, we use $q_i^{\bmu,\pi}$ to denote  $q_i^{D, D_{\text{trig}}, \pi}$ for simplicity.
Triggering probabilities  $q_i^{\bmu,\pi}$'s are crucial for the triggering probability modulated bounded smoothness conditions to be defined below. Second, we define the batch-size $K\overset{\df}{=}\max_{\pi \in \Pi}\sum_{i\in[m]}q_i^{\bmu,\pi}$ as the maximum expected number of arms that can be triggered. Note that this definition is much smaller than $K'\overset{\df}{=}\max_{\pi \in \Pi}\sum_{i\in[m]}\I\{q_i^{\bmu,\pi}>0\}$ originally defined in \cite{wang2017improving}. For example, in episodic RL, this difference saves a factor of $S$, i.e., $K'=SH$ and $K=H$.

Owing to the nonlinearity and the combinatorial structure of the reward, it is essential to give some conditions for the reward function in order to achieve any meaningful regret bounds~\cite{chen2013combinatorial, chen2016combinatorial, wang2017improving, degenne2016combinatorial,merlis2019batch}. 
In this paper, we consider the smoothness condition as follows. 


\begin{condition}[1-norm multivariant and triggering probability modulated (MTPM) smoothness condition]\label{cond:general_smooth}
We say that a CMAB-MT problem satisfies 1-norm MTPM smoothness condition, if there exist weight vectors $\bw_i^{\tilde{\bmu},\pi}\in [0,\bar{w}]^d$ for $\tilde{\bmu}\in[0,1]^{m\times d},\pi\in\Pi,i\in[m]$ such that, for any two distributions $\tilde{D},D\in \cD$ with mean $\tilde{\bmu}, \bmu \in [0,1]^{m \times d}$, and for any action $\pi\in \Pi$, we have
\begin{align}\label{eq:abs_smooth}
\abs{r(\pi;\tilde{\bmu})-r(\pi;\bmu)}\le \sum_{i\in [m]}q_i^{\bmu,\pi} \abs{\abs{\tilde{\bmu}_i- \bmu_i}^{\top}\bw_i^{\tilde{\bmu},\pi}}.
\end{align}
\end{condition}
Furthermore, if for any two distributions $\tilde{D},D\in \cD$ with mean $\tilde{\bmu}, \bmu \in [0,1]^{m \times d}$, and for any action $\pi\in \Pi$ we have
\begin{align}\label{eq:general_smooth}
    \abs{r(\pi;\tilde{\bmu})-r(\pi;\bmu)}\le \sum_{i\in [m]}q_i^{\bmu,\pi}  \abs{(\tilde{\bmu}_i-\bmu_i)^{\top}\bw_i^{\tilde{\bmu},\pi}},
\end{align}
then we say CMAB-MT problem satisfies 1-norm MTPM+ smoothness condition.


\begin{remark}[Intuitions of Condition~\ref{cond:general_smooth}]The 1-norm MTPM smoothness condition aims to bound the reward difference caused by the parameter changing from $\bmu$ to $\tilde{\bmu}$. Intuitively, we use $\abs{\tilde{\bmu}_i- \bmu_i}^{\top}\bw_i^{\tilde{\bmu},\pi}$ to characterize the parameter difference for each multivariant base arm $i$. Instead of directly using the 1-norm distance $\norm{\bmu_i-\tilde{\bmu}_i}_1$, we use a refined weighted 1-norm where each dimension's difference, each dimension's difference $\abs{\mu_{i,j}-\tilde{\mu}_{i,j}}$ is weighted by $w_{i,j}^{\tilde{\bmu},\pi}$ for $j\in[d]$. 
Then, each arm $i$'s parameter difference is re-weighted by the triggering probability $q_i^{\bmu,\pi}$. Intuitively, for base arm $i$ that is unlikely to be triggered/observed (small $q_i^{\bmu,\pi}$), Condition~\ref{cond:general_smooth} ensures that a large change in $\bmu_i$ 
only causes a small change (multiplied by $q_i^{\bmu,\pi}$) in the reward, and thus one does not need to pay extra regret to observe such arms. 
 Notice that $q_i^{\bmu,\pi}$ and $\bw_i^{\tilde{\bmu},\pi}$ are related to $\bmu$ and $\tilde{\bmu}$, respectively, to keep the balance of $\bmu,\tilde{\bmu}$, and the condition still holds if one exchanges the $\bmu$ and $\tilde{\bmu}$ of $q_i^{\bmu,\pi}$ and $\bw_i^{\tilde{\bmu},\pi}$.
The intuition for 1-norm MTPM+ is similar to 1-norm MTPM, but
1-norm MTPM+ condition is stronger since any CMAB-MT instance satisfies 1-norm MTPM+ condition with $\bw_i^{\tilde{\bmu},\pi}\in [0,\bar{w}]^d$ also satisfies 1-norm MTPM condition with the same $\bw_i^{\tilde{\bmu},\pi}$ , owing to the fact that $\abs{(\tilde{\bmu}_i- \bmu_i)^{\top}\bw_i^{\tilde{\bmu},\pi}}\le \abs{\tilde{\bmu}_i- \bmu_i}^{\top}\bw_i^{\tilde{\bmu},\pi}$. In \cref{sec:RL}, we will show episodic RL satisfies the stronger 1-norm \MTPMr, which is the key to achieving minimax regret bound. On the other hand, the weaker 1-norm MTPM condition is easier to satisfy and can potentially cover more applications. 
\end{remark}

\begin{remark}[Instances and extensions of 1-norm MTPM condition]\label{rmk:CMAB-T instance}
1-norm MTPM/MTPM+ smoothness condition reduces to $B_1$ bounded 1-norm TPM condition in the CMAB-T framework \cite{wang2017improving} when $d=1$ and $w_i^{\tilde{\bmu},\pi}=B_1$, covering all CMAB-T problems as instances. 
More importantly, the 1-norm MPTM/MTPM+ smoothness condition covers the smoothness of value functions for episodic RL in \cref{sec:RL_sec_loose} and \cref{sec:tight_RL}.
1-norm MTPM/MTPM+ can also be transformed to the infinity norm or other norms, e.g., $\abs{r(\pi;\tilde{\bmu})-r(\pi;\bmu)}\le \max_{i\in [m]}q_i^{\bmu,\pi} \cdot \abs{\tilde{\bmu}_i-\bmu_i}^{\top} \bw_i^{\tilde{\bmu},\pi}$. 
When $d=1$ and $w_i^{\tilde{\bmu},\pi}=B_{\infty}$, $\ell_\infty$-norm MTPM reduces to the $B_{\infty}$-bounded max-norm smoothness condition \cite{chen2013combinatorial}. 
\end{remark}



    

\section{Multivariant CUCB Algorithm}
For CMAB-MT, we give a combinatorial upper confidence bound algorithm with a general joint oracle (CUCB-MT) in \Cref{alg:CUCB-MT}. 
Since the CUCB-MT algorithm and its analysis are slightly different for problems satisfying the 1-norm MPTM condition in \cref{eq:abs_smooth} and the stronger 1-norm MPTM+ condition in \cref{eq:general_smooth}, we unify the notation and for any $\bu,\bv\in \R^d$ use $[\bu-\bv]_+\overset{\df}{=}\abs{\bu-\bv}$ for the former problems and $[\bu-\bv]_+\overset{\df}{=}(\bu-\bv)$ for the latter problems.

CUCB-MT utilizes the principle of optimism in the face of uncertainty.
In each round $t$, it first constructs a function $\cC$, where $\cC_t(\pi) \subseteq [0,1]^{m\times d}$ are action-dependent confidence region around the empirical mean $\hat{\bmu}_{t-1}$.
In this work, we assume the $\cC_t(\pi)$ for any $\pi\in\Pi$ is defined as: 
\begin{align}\label{eq:conf_set}
   \cC_t(\pi)\overset{\df}{=}\left\{\tilde{\bmu}:\abs{\uabs{\tilde{\bmu}_i- \hat{\bmu}_{t-1,i}}^{\top}\bw_{i}^{\tilde{\bmu},\pi}}\le \phi_{t,i}, \forall i\in[m]\right\},
\end{align}
where $\bw_i^{\tilde{\bmu},\pi}$ are weights specified by Condition~\ref{cond:general_smooth}, $\phi_{t,i}$ are confidence radius defined as $\phi_{t,i}= F_{t,i}\sqrt{\frac{1}{N_{t-1,i}}}+I_{t,i}\frac{1}{N_{t-1,i}}$, and $F_{t,i}, I_{t,i}$ are problem-specific values that will be specified for subsequent applications.

CUCB-MT then selects an optimistic action-parameter pair $(\pi_t,\tilde{\bmu}_t)$ with the help of the joint oracle $\tilde{\cO}$. 
Note that the joint oracle is determined by the confidence region $\cC_t$ and the reward function $r(\pi;\bmu)$. 
The agent then plays the selected combinatorial action $\pi_t$, which will trigger a set of multivariant base arms $\tau_t$ whose $d$-dimensional outcomes are observed. Finally, CUCB-MT updates the statistics and historical information accordingly to improve future decisions.
Note that though the form of $\cC_t$ and the joint oracle are abstracted out in CUCB-MT, we will give concrete applications as examples in \Cref{sec:RL} and \cref{sec:beyond_RL}, where determining $\cC_t$ and $\tilde{\cO}$ serve as key ingredients of efficient algorithms and tight regret bounds.

\begin{algorithm}[t]
	\caption{CUCB-MT: Combinatorial Upper Confidence Bound Algorithm for CMAB-MT}\label{alg:CUCB-MT}
			\resizebox{.94\columnwidth}{!}{
\begin{minipage}{\columnwidth}
	\begin{algorithmic}[1]
	    \STATE {\textbf{Input:}} Base arms $[m]$, dimension $d$, joint oracle $\tilde{\cO}$.
	   \STATE \textbf{Initialize:} For each arm $i$, $N_{0,i} = 0, \hat{\bmu}_{0,i}=\boldsymbol{0}$. 
	   \FOR{$t=1, ...,T$ }
    \STATE Construct an action-dependent confidence region function $\cC_t$ around $\hat{\bmu}_{t-1}$ according to \Cref{eq:conf_set}.
	   \STATE Apply joint oracle $\tilde{\cO}$ and get $(\pi_t,\tilde{\bmu}_t)=\tilde{\cO}(\cC_t)$.\label{line:joint_oracle}
	   \STATE Play action $\pi_t$, which triggers arms $\tau_t \subseteq [m]$ with outcome $\bX_{t,i}$, for $i \in \tau_t$.\label{line:trigger_obs}
     \STATE For $i \in \tau_t$, update $N_{t,i}= N_{t-1, i}+1$, $\hat{\bmu}_{t,i}= \hat{\bmu}_{t-1,i}+(\bX_{t,i}-\hat{\bmu}_{t-1,i})/N_{t,i}$. For $i \not\in \tau_t$, keep $N_{t,i}= N_{t-1, i}$, $\hat{\bmu}_{t,i}= \hat{\bmu}_{t-1,i}$.
	    \label{line:trigger_upd}
	    	   \ENDFOR
		\end{algorithmic}
		\end{minipage}}
\end{algorithm}


\subsection{Analysis of CUCB-MT and Its Discussion}

Fix the underlying distribution $D \in \cD$ and its mean vector $\bmu\in [0,1]^{m\times d}$ with optimal action $\pi^*$. For each action $\pi \in \Pi$, we define the (approximation) gap as $\Delta_{\pi}=\max\{0, \alpha r(\pi^*;\bmu)-r(\pi;\bmu)\}$. 
For each arm $i\in[m]$, we define $\Delta_i^{\min}=\inf_{\pi\in \Pi: q_i^{\bmu,\pi} > 0,\text{ } \Delta_{\pi} > 0}\Delta_{\pi}$. 

Recall that for any round $t$, $\hat{\boldsymbol{\mu}}_{t-1,i}$ is the empirical mean, $\mathcal{C}_t$ is the confidence region function defined in \cref{eq:conf_set} with problem-specific parameters $F_{t,i}$ and $I_{t,i}$, $\tilde{\mathcal{O}}$ is the joint oracle, and $(\pi_t,\tilde{\boldsymbol{\mu}}_{t})=\tilde{\mathcal{O}}(\mathcal{C}_t)$ are the pair of optimistic policy and parameter in line 5 of \Cref{alg:CUCB-MT}. Define the concentration event 
\begin{align}\label{eq:thm_concen_1}
	\mathcal{E}_{c,1} = \Big\{ &\abs{\uabs{\boldsymbol{\mu}_i-\hat{\boldsymbol{\mu}}_{t-1,i}}^{\top}\boldsymbol{w}^{\boldsymbol{\mu},\pi^*}_i} \le 
	 F_{t,i}\sqrt{\frac{1}{N_{t-1,i}}}\notag\\
  &+I_{t,i}\frac{1}{N_{t-1,i}}, \text{ for all } i\in[m], t\in[T] \Big\}.
\end{align}
Let $G_{t,i}, J_{t,i}$ be another two problem-specific parameters for $i\in[m],t\in[T]$, and define the second concentration event 
\begin{align}\label{eq:thm_concen_2}
    &\mathcal{E}_{c,2} = \Big\{\abs{\uabs{\boldsymbol{\mu}_i-\hat{\boldsymbol{\mu}}_{t-1,i}}^{\top}(\boldsymbol{w}^{\tilde{\boldsymbol{\mu}}_t,\pi_t}_i-\boldsymbol{w}^{\boldsymbol{\mu},\pi^*}_i)}\notag\\&\le G_{t,i}\sqrt{\frac{1}{N_{t-1,i}}}+J_{t,i}\frac{1}{N_{t-1,i}}, \text{ for all } i\in[m], t\in[T] \Big\}.
\end{align} 
Let $\bar{F},\bar{G}, \bar{I},\bar{J}$ be upper bounds for problem specific parameters so that $\sum_{i\in[m]}q_i^{\bmu,\pi_t}F_{t,i}^2\le \bar{F}$, $\sum_{i\in[m]}q_i^{\bmu,\pi_t}G_{t,i}^2\le \bar{G}$, $I_{t,i}\le \bar{I}$, $J_{t,i}\le \bar{J}$, then we have the following theorem. 

\begin{theorem}\label{thm:reg_CMAB-MT}
    For a CMAB-MT problem instance $([m], d, \Pi, \cD, D_{\text{trig}},R)$ that satisfies 1-norm MTPM or MTPM+ smoothness condition (Condition \ref{cond:general_smooth}) with weight vectors $\bw_i^{\tilde{\bmu},\pi}\in [0,\bar{w}]^d$ for $\tilde{\bmu}\in[0,1]^{m\times d},\pi\in\Pi,i\in[m]$,
    if the oracle $\tilde{\cO}$ is an $(\alpha,\beta)$-approximation oracle, and concentration events $\mathcal{E}_{c,1},\mathcal{E}_{c,2}$ hold with probability at least $1-\frac{1}{T}$, then CUCB-MT (Algorithm 1) 
    achieves an $(\alpha,\beta)$-approximate gap-dependent regret bounded by
    \begin{align}
        O\left(\sum_{i \in [m]} \frac{(\bar{F}+\bar{G})}{\Delta_{i}^{\min} } + (\bar{I}+\bar{J})  \log\left (\frac{(\bar{I}+\bar{J})K}{\Delta_{i}^{\min}}\right)\right),
    \end{align}
and the gap-independent regret bounded by
\begin{align}
    O\left(\sqrt{m(\bar{F}+\bar{G})T}+m(\bar{I}+\bar{J})\log (KT)\right).
\end{align}
\end{theorem}

\begin{remark}
Looking at the above theorem, problem-specific parameters $F_{t,i},I_{t,i},G_{t,i},J_{t,i}$ are related to concentration inequalities that hold with high probability, so they are polylogarithmic terms regarding $T$. For example, when arms are $d$-dimensional multinoulli random variables, $F_{t,i},G_{t,i}=O(\bar{w}\sqrt{d\log T})$ and $I_{t,i}=J_{t,i}=0$. Therefore, the leading regret is $O(\sqrt{m(\bar{F}+\bar{G})T})$. For event $\cE_{c,1}$ in \cref{thm:reg_CMAB-MT}, we only require $\bmu\in \cC_{t}(\pi^*)$ (instead of $\bmu\in \cC_{t}(\pi)$ for all $\pi$), which can obtain smaller $F_{t,i},I_{t,i}$ with tighter regret. For event $\cE_{c,2}$ in \cref{thm:reg_CMAB-MT}, $G_{t,i},J_{t,i}$ can be very small since $\bw^{\tilde{\bmu}_t,\pi_t}_{i},\bw^{\bmu,\pi^*}_{i}$ can be very close to each other, e.g., if $\bw_i^{\bmu,\pi}=\bc$ are constant vectors, then $G_{t,i}=J_{t,i}=0$. 
Now looking at any CMAB-T problem following 1-norm TPM condition with $B_1$ mentioned in \cref{rmk:CMAB-T instance} and arms being Bernoulli, $F_{t,i}=B_1\sqrt{1.5\log (mT)},G_{t,i}=I_{t,i}=J_{t,i}=0$ and our theorem gives $\tilde{\cO}(B_1\sqrt{mKT})$ regret, matching the tight regret bound given by \cite{wang2017improving}. In later sections, we will provide two representative applications that fit into the CMAB-MT framework and identify parameters $\bar{F},\bar{G},\bar{I},\bar{J}$, which achieves matching or improved regret bounds compared to existing works. Due to the space limit, the detailed analysis of \cref{thm:reg_CMAB-MT} is deferred to  \cref{apdx_sec:proof_of_CMAB-MT}.

\end{remark}

\section{Application to Episodic Reinforcement Learning}\label{sec:RL}
In this section, we first introduce the setup of episodic RL, which is modeled as a finite-horizon Markov decision process (MDP). Then we demonstrate how episodic RL can be effectively integrated into the framework of CMAB-MT and satisfy two different 1-norm MTPM smoothness conditions. For the former, we give a result matching that of the seminal work~\cite{jaksch2010near} as a warm-up case. For the latter, we achieve the minimax-optimal regret bound by leveraging a tighter confidence region function and the variance-aware analysis.

\subsection{Setup of Episodic MDP and RL}
We consider the finite-horizon MDP, i.e., episodic MDP, which can be described by a tuple $(\cS, \cA, H, \cP, \cR)$. $\cS$ is the finite state space with cardinality $S$. $\cA$ is the finite action space with cardinality $A$. $H$ is the number of steps for each episode. $\cP=(\bp(s,a,h))_{(s,a,h) \in \cS \times \cA \times [H]}$ are transition kernels, where $\bp(s,a,h)\in \bDelta_{S}$ and each element $p(s'|s,a,h)$ is the probability of transitioning to state $s'$ after taking action $a$ in state $s$ at step $h$.\footnote{We consider the time inhomogeneous setting where $\bp(s,a,h)$'s at different steps are different.} For the ease of exposition, reward distribution $\cR=(r(s,a,h))_{(s,a,h) \in \cS \times \cA \times [H]}$ are assumed to be Bernoulli random variables with mean $r(s,a,h)\in [0,1]$, indicating the instantaneous reward collected upon taking action $a$ in state $s$ at step $h$.

In episodic RL, the agent interacts with an unknown episodic MDP environment (where $\cP$ and $\cR$ are unknown)  in a sequence of episodes $t \in [T]$. At the beginning of episode $t$, the agent starts from a fixed initial state $s_{1}$ and determines a policy $\pi_t$, where $\pi_{t}(s,h) \in \cA$ maps any state and step $h$ to actions.\footnote{The fixed initial state can be generalized to random initial state $s_{t,1}$ by using a $H+1$ step MDP which virtually starts from a fixed state $s_0$ and transits to $s_{t,1}$ with (unknown) distribution $\bp_0$.} 
Then at the step $h=1,..., H$, the agent selects an action $a_{t,h}=\pi_{t}(s_{t,h},h)$,
receives a random reward $R_{k,h} \in [0,1]$ with mean $r(s_{t,h},a_{t,h},h)$, and transits to the next state $s_{t,h+1}$ with probability $p(s_{t,h+1}\mid s_{t,h}, a_{t,h}, h)$. The trajectory $(s_{t,h}, a_{t,h})_{h \in [H]}$ and the random reward $(R_{k,h})_{h \in [H]}$ are observed as feedback to improve future policies. 

Each policy $\pi$ specifies a value function for every state $s$ and step $h$ (i.e., the expected total reward starting from state $s$ at step $h$ until the end of the episode), defined as $V_{h}^{\pi}(s)=\E[\sum_{i=h}^H r(s_{i}, a_i, i)\mid s_h=s, \pi]$, where the expectation is taken over visited state-action pairs $(s_i, a_i)$ upon starting from state $s$ at step $h$. It is easily shown that the value function satisfies the Bellman equation (with $\bV_{H+1}^{\pi}=\boldsymbol{0}$) for any policy $\pi$:
\begin{align}
    &V_h^{\pi}(s)=r(s,\pi(s,h),h) + \bp(s,\pi(s,h),h)^{\top}\bV_{h+1}^{\pi}.
\end{align}
For episodic MDP, there always exists a policy $\pi$ that attains the best possible values, and we define the optimal value function $V^*_h(s)=\sup_{\pi}V_h^{\pi}(s)$.
The objective of episodic RL is to minimize the regret over $T$ episodes, which is defined as 
\begin{align}
    \text{Reg}(T)\overset{\text{def}}{=}\sum_{t \in [T]}\left(V_{1}^{*}(s_{1})-V_{1}^{\pi_t}(s_{1})\right).
\end{align}

\subsection{Episodic RL from the Lens of CMAB-MT}
Similar to existing works, we assume transition kernels $\cP$ are unknown while the reward distribution $\cR$ is known.\footnote{Handling unknown reward distribution $\cR$ is straight-forward by adding $SAH$ arms with $d=1$ for the $SAH$ unknown rewards.}
For this episodic RL problem, it fits into CMAB-MT framework with tuple $([m],d, \Pi, \cD, D_{\text{trig}},R)$. Each transition kernel $\bp(s,a,h) \in \bDelta_S$ corresponds to a base arm and there are $m=SAH$ of them. The outcome of base arm $\bX_{s,a,h}\in \{0,1\}^S$ is a multinoulli (or categorical) random variable with dimension $d=S$, i.e., a one-hot vector $\bX_{s,a,h}=\be_{s'}$ indicating the state at next step $h+1$ will be $s'$ upon taking action $a$ at step $h$. The set of feasible combinatorial actions $\pi$ corresponds to the set of deterministic policies $\pi$ that maps state-step pairs to actions, i.e., $\pi: \cS \times [H] \rightarrow \cA$. As mentioned before, we assume the reward distribution $\cR$ is known, so the set of $\cD$ corresponds to any feasible MDP with reward distribution $\cR$. 

Before the RL game starts, the environment draws an unknown distribution $D\in \cD$ with transition probabilities $\bp=(\bp(s,a,h))_{s,a,h\in \cS\times\cA\times[H]}$, where $\bp(s,a,h)\in \bDelta_S$. 
At each episode $t\in[T]$, let the outcomes of base arms be $\bX_t=(\bX_{t,s,a,h})_{s,a,h}\sim D$. 
Given the policy $\pi_t$ and the starting state $s_1$, the triggering set $\tau_t=(s_{t,h},a_{t,h},h)_{h\in [H]}$ includes a cascade of $H$ base arms starting from the state-action-step tuple $(s_1,\pi_t(s_1,1),1)$, and the $h$-th arm for $h>1$ of this cascade is tuple $(s_{t,h}, a_{t,h},h)=(s',\pi_t(s',h),h)$, where $s'\in \cS$ is the index such that $s'$-th entry of $\bX_{t, s_{t,h-1},a_{t,h-1},h-1}$ equals to $1$. 
In this case, the triggering probability distribution $D_{\text{trig}}(\pi_t,\bX_t)$ is fully determined by  $\pi_t$ and $\bX_t$, i.e., $\tau_t$ is deterministically decided given $\pi_t$ and $\bX_t$. And it is easy to show that the reward function $R(\pi_t,\bX_t, \tau_t)=\sum_{(s,a,h)\in \tau_t}r(s,a,h)$ and the expected reward function $r(\pi_t;\bp)=\E[R(\pi_t,\bX_t, \tau_t)]=V_1^{\pi_t}(s_1)$. 

\textbf{Key Quantities.} 
The triggering probability is the occupancy measure of $(s,a,h)$, i.e., $q_{s,a,h}^{\bp,\pi}=\E[\I\{s_h=s, a_h=a\mid \pi, \bp\}]$, indicating the probability of visiting state-action pair $(s,a)$ at step $h$ when the underlying transition is $\bp$ and the policy is $\pi$. And the batch-size is $K=\max_{\pi}\sum_{s,a,h}q_{s,a,h}^{\bp,\pi}=H$.

\subsection{The Simple Smoothness Condition with Constant Weights Achieves Sublinear Regret}\label{sec:RL_sec_loose}
Fitting episodic RL into CMAB-MT, we can show that it satisfies the following lemma, whose proof is in \cref{apdx_sec:proof_lem1n2}.
\begin{lemma}\label{lem:RL_smooth_loose}
    Episodic RL with unknown transition is a CMAB-MT instance, which satisfies 1-norm MTPM smoothness condition (Condition \ref{cond:general_smooth}) with weights  $\bw_{s,a,h}^{\tilde{\bp},\pi}=H\cdot\boldsymbol{1}\in \R^{S}$ for all $\tilde{\bp},\pi$, i.e., $
    \abs{V_1^{\tilde{\bp},\pi}(s_1)-V_1^{\bp,\pi}(s_1)}\le H\cdot\sum_{s,a,h}q_{s,a,h}^{\bp, \pi} \cdot \norm{\tilde{\bp}(s,a,h)-\bp(s,a,h)}_1$.
\end{lemma}

\textbf{Confidence Region Function $\cC_t$ and Joint Oracle $\tilde{\cO}$.} By definition, we have counter $N_{t}(s,a,h)=\sum_{t'=1}^{t}\I\{(s,a,h)\in \tau_{t'}\}$ 
and the empirical mean $\hat{p}_{t}(s,a,h)=\frac{\sum_{t'=1}^{t}\I\{(s,a,h)\in \tau_{t'}\}\bX_{t',s,a,h}}{N_{t}(s,a,h)}$. Based on the fact that outcomes $\bX_{t,s,a,h}$ are multinoulli random variables, we use the concentration for multinoulli distributions (\cref{apdx_lem:concen_tran}), i.e., $\norm{\bp(s,a,h)-\hat{\bp}_{t-1}(s,a,h)}_1\le \sqrt{\frac{2S\log (2/\delta)}{N_{t-1}(s,a,h)}}$ with probability at least $1-\delta$. The confidence region function defined as \cref{eq:conf_set} becomes 
\begin{align}\label{eq:RL_loose_conf_set}
    &\cC_t(\pi)=\{\tilde{\bp}:\text{ for all } (s,a,h), \tilde{\bp}(s,a,h)\in \bDelta_{S},\notag\\
    &H\cdot\norm{\tilde{\bp}(s,a,h)-\hat{\bp}_{t-1}(s,a,h)}_1\le\phi_t(s,a,h) \},
\end{align} 
where $\phi_t(s,a,h)=F_{t,s,a,h}\sqrt{\frac{1}{N_{t-1}(s,a,h)}}+\frac{I_{t,s,a,h}}{N_{t-1}(s,a,h)}$, and $F_{t,s,a,h}=H\sqrt{2S\log (SAHT/\delta')}$, $I_{t,s,a,h}=0$.

Since this region is not policy-dependent, we use $\cC_t$ as a shortcut of $\cC_t(\pi)$ for all $\pi\in \Pi$. 
Based on the confidence region $\cC_t$, we identify the joint oracle as $(\pi_t,\tilde{\bp}_t)=\argmax_{\pi\in \Pi, \tilde{\bp}\in \cC_t}V_1^{\tilde{\bp},\pi}(s_1)$. According to \citet{jaksch2010near}, this joint oracle can be implemented efficiently using extended value iteration described in \cref{alg:ext_VI}. Note that in line~\ref{line:alg2_max} in \cref{alg:ext_VI}, we need to solve a linear optimization problem over a convex polytope, which can be solved in $O(S^2A)$ \cite{jaksch2010near}.
\begin{algorithm}[t]
	\caption{Extended Value Iteration Oracle in Episode $t$}\label{alg:ext_VI}
		\resizebox{.93\columnwidth}{!}{
\begin{minipage}{\columnwidth}
	\begin{algorithmic}[1]
	    \STATE \textbf{Input:} Counter $N_{t-1}(s,a,h)$, empirical transition $\hat{\bp}_{t-1}(s,a,h)$ for all $s,a,h$, and  $\delta'=1/(2T)$.  
	   \STATE \textbf{Initialize:} $\phi'_t(s,a,h)=\sqrt{\frac{2S\log (SAHT/\delta')}{N_{t-1}(s,a,h)}}$, $\bar{V}_{t,H+1}(s)=0$ for all $s,a,h$.
    \FOR{Step $h=H, H-1, ..., 1$}
    \STATE For all $(s,a)$, set $\tilde{\bp}_t(s,a,h)=\argmax_{\bp'\in\bDelta_{S}:\norm{\bp'-\hat{\bp}_{t-1}(s,a,h)}_1\le \phi'_t(s,a,h)}\bp'^{\top}\bar{\bV}_{t,h+1}$\label{line:alg2_max}
    \STATE For all $(s,a)$, set $Q_t(s,a,h)=r(s,a,h)+\tilde{\bp}_t(s,a,h)^{\top}\bar{\bV}_{t,h+1}$.
    \STATE For all $s$, set $\pi_t(s,h)=\argmax_aQ_t(s,a,h)$ and $\bar{V}_{t,h}(s)=Q_t(s,\pi_t(s,h),h)$
    \ENDFOR
    \STATE \textbf{Return:} $\pi_t,\tilde{\bp}_t$.
		\end{algorithmic}   
				\end{minipage}}
\end{algorithm}

\textbf{Regret Bound and Discussion.} 
Based on the above confidence region function and the joint oracle, we have
\begin{theorem}\label{thm:reg_RL_loose}
     For episodic RL fitting into the CMAB-MT framework  with weights in 
     \cref{lem:RL_smooth_loose}, CUCB-MT algorithm with the confidence region function $\cC_t$ in \cref{eq:RL_loose_conf_set} and the joint oracle in \cref{alg:ext_VI} satisfies the requirements of \cref{thm:reg_CMAB-MT} with parameters $\bar{F}=\tilde{O}(H^3S),\bar{G}=\bar{I}=\bar{J}=0$, and thus achieves a regret bounded by $\tilde{O}(\sqrt{H^4S^2AT})$.
\end{theorem}

Looking at the above theorem, we achieve a regret bound matching the seminal work \citet{jaksch2010near}, and up to a factor of $\tilde{O}(\sqrt{HS})$ compared with lower bound given by \citet{jaksch2010near}, see \cref{apdx_sec:RL_loose_proof} for detailed analysis.

\subsection{The Value Function Related Smoothness Condition Achieves Optimal Regret}\label{sec:tight_RL}
\begin{algorithm}[t]
	\caption{Optimistic Value Iteration Oracle in Episode $t$}\label{alg:v-opt}
		\resizebox{.93\columnwidth}{!}{
\begin{minipage}{\columnwidth}
	\begin{algorithmic}[1]
	    \STATE \textbf{Input:} Counter $N_{t-1}(s,a,h)$, empirical transition $\hat{\bp}_{t-1}(s,a,h)$ for all $s,a,h$, and $\delta'=1/(8T)$.  
	   \STATE \textbf{Initialize:} Constant $L=\logL$, value function $\bar{V}_{t,H+1}(s)=\ubar{V}_{t,H+1}(s)=0$, for all $s$.
    \FOR{ $h=H, H-1, ..., 1$}
\STATE For all $(s,a)$, set confidence radius $\phi_{t}(s,a,h)=2\sqrt{\frac{\Var_{s'\sim \hat{\bp}_{t-1}(s,a,h)}\left(\bar{V}_{t,h+1}(s')\right)L}{N_{t-1}(s,a,h)}}+2\sqrt{\frac{\E_{s'\sim \hat{\bp}_{t-1}(s,a,h)}\left[\bar{V}_{t,h+1}(s')-\ubar{V}_{t,h+1}(s')\right]^2L}{N_{t-1}(s,a,h)}}+\frac{5HL}{N_{t-1}(s,a,h)}$.\label{line:onp_conf} 
    \STATE Set $s^*=\argmax_{s}\bar{V}_{t,h+1}(s)$.
    \FOR{$(s,a)\in \cS\times \cA$}
    \IF {$\bar{V}_{t,h+1}(s^*)< \hat{\bp}_{t-1}(s,a,h)^{\top}\bar{\bV}_{t,h+1}+\phi_t(s,a,h)$}
    \STATE Set $\tilde{\bp}_{t}(s,a,h)=\be_{s^*}$
    \ELSE
    \STATE Pick any $\tilde{\bp}_t(s,a,h)\in \bDelta_S$ s.t. $\tilde{\bp}_t(s,a,h)^{\top}\bar{\bV}_{t,h+1}=\hat{\bp}_{t-1}(s,a,h)^{\top}\bar{\bV}_{t,h+1}+\phi_t(s,a,h)$.
    \ENDIF
    \STATE $Q_t(s,a,h)=r(s,a,h)+\tilde{\bp}_t(s,a,h)^{\top}\bar{\bV}_{t,h+1}$. 
    \ENDFOR
    \STATE For all $s$, set $\pi_t(s,h)=\argmax_{a}Q_t(s,a,h)$.
    \STATE For all $s$, set $\bar{V}_{t,h}(s)=Q_t(s,\pi_t(s,h),h)$.
    \STATE For all $s$, set $\ubar{V}_{t,h}(s)=\max\{r(s,\pi_t(s,h),h)+\hat{\bp}_{t-1}(s,\pi_t(s,h),h)^{\top}\ubar{V}_{t,h+1}-\phi_t(s,a,h),\,\,0\}$.
    \ENDFOR
    \STATE \textbf{Return:} $\pi_t,\tilde{\bp}_t$. 
		\end{algorithmic}   
				\end{minipage}}
\end{algorithm}
A natural question to ask is whether we can achieve the minimax optimal regret using the CMAB-MT framework. The answer is affirmative by leveraging the RL structures for stronger 1-norm MTPM+ smoothness condition, tighter confidence region $\cC_t$, and variance-aware analysis. 
We start with a stronger smoothness condition. Compared with \cref{lem:RL_smooth_loose}, we use the future value function $\bV_{h+1}^{\tilde{\bmu},\pi}$ instead of the constant $H\cdot\boldsymbol{1}$ as the weight $\bw_{s,a,h}^{\tilde{\bmu},\pi}$, whose proof is in \cref{apdx_sec:proof_lem1n2}. 
\begin{lemma}\label{lem:RL_tight_smooth}
    Episodic RL with unknown transition is a CMAB-MT instance, which satisfies 1-norm MTPM+ smoothness (Condition \ref{cond:general_smooth}) with weight vector $\bw_{s,a,h}^{\tilde{\bp},\pi}=\bV_{h+1}^{\tilde{\bp},\pi}$ for all $\tilde{\bp},\pi$,  i.e., $
    \abs{V_1^{\tilde{\bp},\pi}(s_1)-V_1^{\bp,\pi}(s_1)}\le \sum_{s,a,h}q_{s,a,h}^{\bp, \pi} \cdot \abs{[\tilde{\bp}(s,a,h)-\bp(s,a,h)]^{\top} \bV_{h+1}^{\tilde{\bp},\pi}}$.
\end{lemma}

\textbf{Confidence Region Function $\cC_t$ and Joint Oracle $\tilde{\cO}$.}
Based on the above lemma, we use the following confidence region that bounds the expected future value $\tilde{\bp}(s,a,h)^{\top}\bV^{\tilde{\bp},\pi}_{h+1}$ around the empirical future value $\hat{\bp}(s,a,h)^{\top}\bV^{\tilde{\bp},\pi}_{h+1}$, i.e.,
\begin{align}\label{eq:RL_conf_set_tight}
    &\cC_t(\pi)=\{\tilde{\bp}: \text{ for all 
  } (s,a,h), \tilde{\bp}(s,a,h)\in \bDelta_{S},\notag\\
    &\abs{[\tilde{\bp}(s,a,h)-\hat{\bp}_t(s,a,h)]^{\top}\bV_{h+1}^{\tilde{\bp},\pi}}\le \phi_{t}(s,a,h)\}.
\end{align}
where $\phi_{t}(s,a,h)$ is the confidence radius to be determined later on. According to \citet{dann2017unifying} (Lemma D.1 in particular), we can show that the exact joint oracle over $\cC_t$, i.e., $(\pi_t,\tilde{\bp}_t)=\argmax_{\pi, \tilde{\bp}\in \cC_t(\pi)}V_1^{\tilde{\bp},\pi}(s_1)$, is optimistic value iteration with bonus $\phi_{t}(s,a,h)$ described in \cref{alg:v-opt}.
For the value of $\phi_t(s,a,h)$, we only need $\bp\in \cC_t(\pi^*)$ as specified in \cref{thm:reg_CMAB-MT}, so one possibility is to set $\phi_t(s,a,h)=\tilde{O}(\sqrt{{\Var_{s'\sim \bp(s,a,h)}\left(V_{h+1}^*(s')\right)}/{N_{t-1}(s,a,h)}})$ according to the concentration of optimal future value (\cref{apdx_lem:concen_future}), saving a factor of $O(\sqrt{S})$ compared to the $\phi_t(s,a,h)$ in \cref{sec:RL_sec_loose}. However, since both $\bp$ and $\bV_{h+1}^*$
are unknown and inspired by \citet{zanette2019tighter}, we use the concentration of \cref{apdx_lem:concen_known_future} and set $\phi_t(s,a,h)$ using optimistic $\bar{\bV}_{t,h}$ and pessimistic $\ubar{\bV}_{t,h}$ as in line~\ref{line:onp_conf} in \cref{alg:v-opt}. Mapping back to the form of \cref{eq:conf_set}, we have $F_{t,s,a,h}=2\sqrt{\Var_{s'\sim \hat{\bp}_{t-1}(s,a,h)}\left(\bar{V}_{t,h+1}(s')\right)L}+2\sqrt{\E_{s'\sim \hat{\bp}_{t-1}(s,a,h)}\left[\bar{V}_{t,h+1}(s')-\ubar{V}_{t,h+1}(s')\right]^2L},I_{t,s,a,h}=5HL$.

\textbf{Regret Bound and Discussion.} 
Based on the above tighter confidence region function and the joint oracle, we have


\begin{theorem}\label{thm:reg_RL_tight}
For episodic RL fitting into the CMAB-MT framework with weight in \cref{lem:RL_tight_smooth}, CUCB-MT algorithm (\cref{alg:CUCB-MT}) with the confidence region function $\cC_t$ in \cref{eq:RL_conf_set_tight} and the joint oracle in \cref{alg:v-opt} achieves a regret bounded by $O(\sqrt{H^3SAT\log (SAHT)}+H^3S^2A\log^{3/2} (SAHT))$ according to the analysis procedure of \cref{thm:reg_CMAB-MT}.
\end{theorem}

Looking at the above regret bound, we obtain a minimax optimal worst-case regret matching the lower bound $\Omega(\sqrt{H^3SAT})$ up to logarithmic factors. Our regret also saves at least a $O(\log^{2}(SAHT))$ factor for the leading $O\left(\sqrt{H^3SAT\log (SAHT)}\right)$ term compared with $O \left(\sqrt{H^3SAT\log^5 (SAHT)}\right)$ regret by \citet{zanette2019tighter} and a $O(\sqrt{\log(SAHT)})$ factor compared with the state of the art result \cite{zhang2021isreinforcement,zhang2023settling}. This is due to our tight analysis that uses sharp CMAB proof techniques and see \cref{sec:tight_RL} for details. 
As a by-product, we give a gap-dependent bound that scales with $O(\log T)$. In the worst case, our result is at most a factor of $O(1/q^*)$ larger than \citet{simchowitz2019non} that uses involved clipping techniques. However, when considering gap-independent bound, ours still improves theirs by a factor of $O(\sqrt{\log (SAHT)})$, see \cref{apdx_sec:proof_gap_dependent_bound} for details. 

\section{Applications Beyond Episodic RL}\label{sec:beyond_RL}\label{sec:CDP}
In this section, we first consider the probabilistic maximum coverage problem for goods distribution (PMC-GD), which is a new variant of the PMC problem~\cite{chen2013combinatorial,merlis2019batch,liu2022batch}. For PMC-GD, we show that CMAB-MT framework can give an improved regret bound compared with using the CMAB-T framework. 


\textbf{Application Setup.} The PMC-GD problem is modeled by a weighted bipartite graph $G=(U,V,E,p)$, 
where $U$ are the nodes to be selected, $V$ are the nodes to be covered, and $E$ are the edges between $U$ and $V$. Each edge $(u,v)$ in $E$ is associated with a probability $p(u,v)$. 
The probability $p(u,v)$ indicates the likelihood that node $u$ from $U$ can cover a target node $v$ in $V$. In the classical PMC problem, each selected node $u'$ can independently cover $v$, and edges $(u',v)$ are independent Bernoulli random variables with mean $p(u',v)$. In goods distribution applications \cite{alkan1991fair,chevaleyre2017distributed}, the good (e.g., food, medicine, product, coupon) given to nodes in $U$ is indivisible and $u$ can only randomly distribute it to exactly one of the target users in $V$. Thus for PMC-GD, each selected node $u'$ will cover one of its neighbors in $V$, and the edges $((u',v))_{v: (u',v)\in E}$ form a multinoulli distribution with $\sum_{v:(u',v)\in E}p(u',v)\le1$. 
The coverage means the target node $v\in V$ receives such indivisible goods. The objective of the decision maker is to select at most $k$ nodes in $U$ to maximize the number of covered nodes in $V$.

For the online PMC-GD problem, we consider $T$ rounds of repeated PMC-GD where the edge probabilities $p(u,v)$'s are unknown initially. Without loss of generality, we assume $G$ is a complete bipartite graph. For each round $t\in [T]$, the agent selects $k$ nodes in $U$ as combinatorial action $\pi_t$, the feedback are $k$ node pairs $(u,v)$, where $v$ receives indivisible goods from $u\in \pi_t$.

\textbf{Fitting into CMAB-MT Framework.} The PMC-GD problem fits into CMAB-MT framework as follows: the nodes $U$ are the set of multivariant base arms, the unknown outcome distribution $D\in \cD$ is the joint of $m = |U|$ multinoulli distribution with dimension $d=|V|$, the vectors $\bmu_i=\bp(i,\cdot)\in \bDelta_V$ are unknown mean vectors for $i\in U$, the set of combinatorial action $\Pi$ are any set of nodes $\pi\subseteq U$ with size $|\pi|\le k$. For the arm triggering in round $t$, the triggering set is $\tau_t=\pi_t$. Let $\bX_t=\{0,1\}^{|U|\times |V|}$ be the random outcome where $X_{t,u,v}=1$ if and only if user $u$ sends the good to user $v$ at time step $t$. The total reward is $R(\pi_t,\bX_t,\tau_t)=\sum_{v \in V}\I\{\exists\, u \in \pi_t \text{ s.t. } X_{t, u,v}=1\}$, and the expected reward $r(\pi_t;\bp)=\sum_{v \in V}\left(1-\prod_{u\in \pi_t}\left(1-p(u,v)\right)\right)$.

\textbf{Key Quantities and Conditions.}
For the triggering probability, $q_i^{\bp,\pi_t}=1$ if $i\in \pi_t$ and $q_i^{\bp,\pi_t}=0$ otherwise. And the batch-size $K=k$.
\begin{lemma}\label{lem:OCD_smooth}
    PMC-GD is a CMAB-MT instance, which satisfies 1-norm MTPM smoothness (\cref{cond:general_smooth}) with weights $\bw_u^{\tilde{\bp},\pi}=\boldsymbol{1}$, i.e.,
    $\abs{r(\pi;\tilde{\bp})-r(\pi;\bp)}\le \sum_{u\in \pi} \norm{\tilde{\bp}(u,\cdot)-\bp(u,\cdot)}_1$.
\end{lemma}

\textbf{Confidence Region Function $\cC_t$ and Joint Oracle $\tilde{\cO}$.}
Since each base arm $u$'s outcome follows from multinoulli distribution, indicating $\norm{\bp(u,\cdot)-\hat{\bp}_{t-1}(u,\cdot)}_1\le\sqrt{\frac{2|V|\log (2/\delta)}{N_{t-1,u}}}$ with probability at least $1-\delta$. The confidence region defined by \cref{eq:conf_set} becomes 
\begin{align}\label{eq:conf_set_PMC_GD}
    \cC_t=\{\tilde{\bp}:&\text{ for any } u\in U, \bp(u,\cdot)\in \bDelta_{V},\notag\\ &\norm{\bp(u,\cdot)-\hat{\bp}_{t-1}(u,\cdot)}_1\le \phi_{t,u}\},
\end{align}
and does not depend on the action $\pi$, where $\phi_{t,u}=F_{t,i}\sqrt{\frac{1}{N_{t-1,u}}}+\frac{I_{t,i}}{N_{t-1,u}}$ and $F_{t,i}=\sqrt{2|V|\log (|U||V|T/\delta')}, I_{t,i}=0$. 

The joint oracle is $(\pi_t,\tilde{\bp}_t)=\argmax_{|\pi|\le k, \tilde{\bp}_t\in\cC_t}r(\pi;\tilde{\bp})$. 
A new challenge arises since the above joint oracle is a hard optimization problem. In particular, $\arg\max_{|\pi|\le k }r(\pi;\tilde{\bp})$ itself is NP-hard given $\tilde{\bp}$, and now we also have to jointly optimize $\tilde{\bp}$ within the confidence region. 
To obtain an efficient oracle, our strategy is to bypass $r(\pi;\tilde{\bp})$ and optimize an upper bound of $r(\pi;\tilde{\bp})$ using \cref{lem:OCD_smooth} for all $\tilde{\bp}$ as the pseudo reward function for PMC-GD:
\begin{align}\label{eq:pseudo_reward}
\bar{r}_t(\pi;\tilde{\bp})=r(\pi;\hat{\bp}_{t-1})+\sum_{u\in \pi}\norm{\tilde{\bp}(u,\cdot)-\hat{\bp}_{t-1}(u,\cdot)}_1
\end{align} 

We now optimize
$(\pi_t,\tilde{\bp}_t)=\argmax_{|\pi|\le k, \tilde{\bp}\in\cC_t}\bar{r}_t(\pi;\tilde{\bp})$, which is solved in \cref{alg:oracle_PMC_GD}. Based on \cref{eq:conf_set_PMC_GD}, first, we can find optimal $\tilde{\bp}_t$ in line~\ref{line:alg3_p} of \cref{alg:oracle_PMC_GD}. 
Then in line~\ref{line:alg3_pi}, we can optimize $\bar{r}_t(\pi;\tilde{\bp}_t)=r(\pi;\hat{\bp}_{t-1})+\sum_{u\in\pi}q_u$ efficiently using a greedy algorithm with $O(k|U|)$ calls to $\bar{r}_t(\pi;\tilde{\bp}_t)$, yielding a $(1-1/e,1)$-approximation since $\bar{r}_t(\pi;\tilde{\bp}_t)$ is a submodular function regarding $\pi\subseteq U$. Since we use pseudo reward $\bar{r}_t(\pi;\bp)$, mapping back the true reward $r(\pi;\bp)$ will have an additional $\sum_{t\in[T]}\bar{r}_t(\pi_t;\bp)-r(\pi_t;\bp)$ term for the final regret, see \cref{apdx_sec:beyond_RL} for details.

\begin{algorithm}[t]
	\caption{Efficient Joint Oracle for PMC-GD in round $t$}\label{alg:oracle_PMC_GD}
		\resizebox{.93\columnwidth}{!}{
\begin{minipage}{\columnwidth}
	\begin{algorithmic}[1]
	    \STATE \textbf{Input:} Counter $N_{t-1,u}$, empirical edge probability $\hat{\bp}_{t-1}(u,\cdot)$ for all $u\in U$, and $\delta'=1/(2T)$.  
	   \STATE \textbf{Initialize:} $\phi_{t,u}=\sqrt{\frac{2|V|\log (|U||V|T/\delta')}{N_{t-1,u}}}$ for all $u\in U$.
    \STATE For all $u \in U$, compute $\tilde{\bp}_t(u,\cdot)=\argmax_{\bp\in \bDelta_{V}:\norm{\bp-\hat{\bp}_{t-1}(u,\cdot)}_1\le \phi_{t,u}}\norm{\bp-\hat{\bp}_{t-1}(u,\cdot)}_1$.\label{line:alg3_p} 
    \STATE For all $u\in U$, set $q_u=\norm{\tilde{\bp}_t(u,\cdot)-\hat{\bp}_{t-1}(u,\cdot)}_1$.
    \STATE $\pi_t=\argmax_{\pi\in\Pi}r(\pi;\hat{\bp}_{t-1})+\sum_{u\in\pi}q_u$.\label{line:alg3_pi}
    \STATE \textbf{Return:} $\pi_t,\tilde{\bp}_t$.
		\end{algorithmic}   
				\end{minipage}}
\end{algorithm}

\textbf{Regret Bound and Discussion.} Based on the above argument, we have the following theorem.
\begin{theorem}\label{thm:reg_PMC_GD}

     For PMC-GD equipped with pseudo-reward in \cref{eq:pseudo_reward}, CUCB-MT algorithm (\cref{alg:CUCB-MT}) with the confidence region function $\cC_t$ in \cref{eq:conf_set_PMC_GD} and the joint oracle in \cref{alg:oracle_PMC_GD} satisfies the requirements of \cref{thm:reg_CMAB-MT} with parameters $\bar{F}=\tilde{O}(k|V|)$, $\bar{G}=\bar{I}=\bar{J}=0$, and thus achieves a $(1-1/e,1)$-approximate regret bounded by $\tilde{O}(\sqrt{k|U||V|T})$.
\end{theorem}
Compared with existing works, our regret improves upon CUCB-T algorithm \cite{wang2017improving} with regret $\tilde{O}(\sqrt{k|U||V|^2T})$ by a factor of $\tilde{O}(\sqrt{V})$, and the recent variance-adaptive algorithms \cite{merlis2019batch,liu2022batch} with regret $\tilde{O}(\sqrt{|U||V|^2T})$ by a factor of $\tilde{O}(\sqrt{|V|/k})$ when $|V|\ge k$, where in most application scenarios $|V|\gg k$ \cite{chen2016combinatorial,liu2023variance}.

\section{Conclusion and Future Directions}\label{sec:conclusion}
In this work, we propose a new combinatorial multi-armed bandit framework with multivariant and probabilistically triggering arms (CMAB-MT). Through our framework, we  build the first connection between episodic RL and CMAB literature, achieving matching or improved results for episodic RL and beyond. For future work, it will be interesting to study the CMAB-MT framework when considering the linear or nonlinear function approximation. One can also explore new application scenarios that can fit into the CMAB-MT framework for improved results.

\section*{Acknowledgements}
The work of Xutong Liu was partially supported by a fellowship award from the Research Grants Council of the Hong Kong Special Administrative Region, China (CUHK PDFS2324-4S04).
The work of Xutong Liu was done during his visit at the University of Massachusetts Amherst.
The work of John C.S. Lui was supported in part by the RGC GRF 14215722. The work of Mohammad Hajiesmaili was supported by CPS-2136199, CNS-2106299, CNS-2102963, CCF-2325956, and CAREER-2045641. The corresponding author Shuai Li is supported by National Science and Technology Major Project (2022ZD0114804) and is partly supported by the Guangdong Provincial Key Laboratory of Mathematical Foundations for Artificial Intelligence (2023B1212010001).

\section*{Impact Statement}
This paper presents a theoretical study on multi-armed bandits and reinforcement learning. There are many potential societal consequences of our work, none of which we feel must be specifically highlighted here.
\bibliography{main}
\bibliographystyle{icml2024}

\appendix

\clearpage
\onecolumn

\section*{Appendix}

\section{Extended Related Works}\label{apdx_sec:related_work}
In this section, we review two lines of literature that are related to this work.

\textbf{Stochastic Combinatorial Multi-Armed Bandits.} There has been a vast literature on stochastic combinatorial multi-armed bandit (CMAB) \cite{gai2012combinatorial,kveton2015tight,combes2015combinatorial,chen2016combinatorial,wang2017improving,merlis2019batch,saha2019combinatorial,liu2022batch}. \citet{gai2012combinatorial} is the first work to consider the stochastic CMAB with semi-bandit feedback. Since then, its algorithm and regret have been improved by \citet{kveton2015tight,combes2015combinatorial,chen2016combinatorial,merlis2019batch} in different settings. To model a broader range of applications, such as online learning to rank \cite{kveton2015cascading,kveton2015combinatorial} and online influence maximization \cite{chen2013information,wen2017online}, \citet{chen2016combinatorial} first generalizes the CMAB to CMAB with probabilistically triggered armed (CMAB-T). Later on, \citet{wang2017improving} improve the regret bound of \citet{chen2016combinatorial} by introducing a new smoothness condition called the triggering probability modulated (TPM) condition, which removes a factor of $1/q^*$ compared to \citet{chen2016combinatorial}, where $q^*$ is the minimum positive probability that any arm can be triggered. Recently, \citet{liu2022batch} introduce a new variance-modulated TPM condition (TPVM) and variance-adaptive algorithms that can further remove a factor of $K$, where $K$ is the number of arms that can be triggered in each round. Beyond these works, \citet{qin2014contextual,li2016contextual,nika2020contextual,demirel2021combinatorial,liu2023contextual,hwang2023combinatorial} study contextual environments with linear/nonlinear base arm structures, \citet{zimmert2019beating,tsuchiya2023further,nie2023framework,wan2023bandit} consider adversarial environments, and \citet{wang2018thompson,huyuk2019analysis,perrault2022combinatorial} investigate Thompson sampling algorithms for both CMAB and CMAB-T settings. However, all the above works assume the outcome of each arm is a uni-variant sub-Gaussian random variable. In this work, we consider a different setting where arms' outcomes are multivariant random variables, and propose a new CMAB-MT framework that can cover new applications, e.g. episodic RL, and give matching/improved regrets by leveraging the statistical properties of the multivariant random variables.  

\textbf{Episodic Reinforcement Learning.} In recent years, there has been an emerging number of works that study provably efficient RL for regret minimization (c.f. \cite{agarwal2019reinforcement}). For episodic RL, the seminal work \cite{jaksch2010near} proposes the UCRL2 algorithm that adds optimistic bonuses on transition probabilities and achieves a regret bound of $\tilde{O}(\sqrt{H^4S^2AT})$, matching the lower bound $\Omega(\sqrt{H^3SAT})$ given by the same work up to a factor of $\tilde{O}(\sqrt{HS})$. Later on, \citet{azar2017minimax} build confidence region directly for value functions rather than transition probabilities and provide a minimax-optimal regret of $\tilde{O}(\sqrt{H^3SAT})$. Their result is then improved by \citet{zanette2019tighter} who proposes an algorithm based on both optimistic and pessimistic values for the bonus design and achieves tighter problem-dependent regret bounds. After this, various works \citep{li2021breaking,zhang2021isreinforcement,menard2021ucb,wu2022nearly,zhang2023settling} refine the lower-order terms of regret. In addition, many studies \citep{jiang@2017,sun2019model,jin2020provably,jin2021bellman,du2021bilinear,zhong2022gec,liu2024maximize,foster2021statistical} extend beyond tabular RL and explore function approximation, although their regret bounds become suboptimal when applied to the tabular setting.
The above works all focus on giving gap-independent regret bound that scales with $\tilde{O}(\sqrt{T})$. There are also other works \cite{simchowitz2019non,dann2021beyond} that focus on studying gap-dependent regret bound that scales with $O(\log T)$ via clipping techniques. 
To the best of our knowledge, we are the first to solve the episodic RL problem by modeling it as a CMAB-MT instance. From this perspective, we propose new algorithms and analysis that achieves the minimax-optimal leading regret with improved logarithmic factors for the leading regret term. 
Our approach also gives gap-dependent regret bounds ``for free" without using the involved clipping techniques, which matches the \citet{simchowitz2019non} up to a factor of $1/q^*$ in the worst case, where $q^*$ is the minimum positive occupancy measure of any state-action pair for any policy. To this end, we build an important connection between the RL and CMAB literature, which may encourage more interactions between these two important directions.

\section{Analysis for CMAB-MT Framework}\label{apdx_sec:CMAB-MT}
\subsection{Definitions}
\begin{definition}[(Approximation) Gap]\label{def:gap}
Fix a distribution $D \in \cD$ and its mean vector $\bmu\in [0,1]^{m\times d}$, for each action $\pi \in \Pi$, we define the (approximation) gap as $\Delta_{\pi}=\max\{0, \alpha r(\pi^*;\bmu)-r(\pi;\bmu)\}$. 
For each arm $i\in[m]$, we define $\Delta_i^{\min}=\inf_{\pi\in \Pi: q_i^{\bmu,\pi} > 0,\text{ } \Delta_{\pi} > 0}\Delta_{\pi}$, $\Delta_i^{\max}=\sup_{\pi\in \Pi: q_i^{\bmu,\pi} > 0, \Delta_{\pi} > 0}\Delta_{\pi}$.
As a convention, if there is no action $\pi\in \Pi$ such that $q_i^{\bmu,\pi}>0$ and $\Delta_{\pi}>0$, 
	then $\Delta_i^{\min}=+\infty, \Delta_i^{\max}=0$. We define $\Delta_{\min}=\min_{i \in [m]} \Delta_i^{\min}$ and $\Delta_{\max}=\max_{i \in [m]} \Delta_i^{\max}$.
\end{definition}

\begin{definition}[Event-Filtered Regret]\label{apdx_def:event_filter_reg} For any series of events $(\cE_t)_{t \in [T]}$ indexed by round number $t$, we define the $Reg^{\text{ALG}}_{\alpha,\bmu}(T,(\cE_t)_{t \in [T]})$ as the regret filtered by events $(\cE_t)_{t \in [T]}$, or the regret is only counted in $t$ if $\cE$ happens in $t$. Formally, 
\begin{align}
    \reg^{\text{ALG}}_{\alpha,\bmu}(T,(\cE_t)_{t \in [T]}) \overset{\df}{=} \E\left[\sum_{t\in[T]}\I(\cE_t)(\alpha\cdot r(\pi^*;\bmu)-r(\pi_t;\bmu))\right].
\end{align}
For simplicity, we will omit $\text{ALG},\alpha,\bmu,t\in[T]$ and rewrite $\reg^{\text{ALG}}_{\alpha,\bmu}(T,(\cE_t)_{t \in [T]})$ as $\reg(T, \cE_t)$ when contexts are clear.
\end{definition}

\subsection{Bounds for event-filtered regrets}
\begin{lemma}[Decomposition of the filtered regret]\label{apdx_lem:decompose_filter_reg}
Let $K\in \mathbb{N}_+.$ For all $t \ge 1$, consider the event
\begin{align}
     \cE_t=\left\{\Delta_{\pi_t}\le \sum_{k\in [K]}R_{t,k}\right\}
\end{align}
and $K$ decomposed events
\begin{align}
    \cE'_{t,k}=\left\{\Delta_{\pi_t}\le KR_{t,k}\right\}
\end{align}
for some $R_{t,k}\ge 0$.
Then, we have 
\begin{align}
    \reg(T,\cE_t)\le \sum_{k\in [K]}\reg(T,\cE'_{t,k})
\end{align}.
\end{lemma}
\begin{proof}
    It suffices to prove that $\I\{\cE_t\}\Delta_{\pi_t}\le \sum_{k\in[K]}\I\{\cE'_{t,k}\}\Delta_{\pi_t}$ for each round $t$. If $\cE_t$ does not hold, we are done. If $\cE_t$ holds, there exists $k'\in[T]$ such that $\Delta_{\pi_t}\le KR_{t,k'}$, so at least one of the $K$ decomposed events holds and $1\le\sum_{k\in[K]}\I\{\Delta_{\pi_t}\le KR_{t,k}\}$, which gives $\I\{\cE_t\}\Delta_{\pi_t}\le\sum_{k\in[K]}\I\{\Delta_{\pi_t}\le KR_{t,k}\}\Delta_{\pi_t}$.
\end{proof}

\begin{lemma}[Null counters]\label{apdx_lem:null_counter}
    For all $i\in[m]$, if there exists constants $K_i\in \R^+$, consider the event 
    \begin{align}
         \cE_t=\left\{\Delta_{\pi_t}\le \sum_{i\in [m]:N_{t-1,i}=0}q_i^{\bmu,\pi_t}K_i\right\}.
    \end{align}
    Then, the event filtered regret $\reg(T, \cE_t)\le \sum_{i\in[m]}K_i$.
\end{lemma}
\begin{proof}
    Let $\cF_{t-1}=((\pi_1, \tau_1, (\bX_{t,i})_{i \in \tau_1}), ..., (\pi_{t-1}, \tau_{t-1}, (\bX_{t,i})_{i \in \tau_{t-1}}), \pi_t)$ be all historical information before $t$ plus the action at $t$, where $\tau_t$ is the triggered arm set at round $t$.
    By the definition of the triggering probability $q_i^{\bmu,\pi_t}$, we have
    \begin{align}
        \reg(T,\cE_t)&=\E\left[\sum_{t\in[T]}\I\{\cE_t\}\Delta_{\pi_t}\right]\\
        &\le \E\left[\sum_{t\in[T]}\E\left[\sum_{i\in[m]}K_i\I\{N_{t-1,i}=0, i\in\tau_t\}\mid \cF_{t-1}\right]\right]\\
        &\le \sum_{i\in[m]}K_i.
    \end{align}
    The last inequality holds since the counter $N_{t-1,i}$ will be added by one if $i\in \tau_t$, indicating the event can only occur for at most one round, giving the upper bound $\sum_{i\in[m]}K_i$.
\end{proof}

\subsection{Proof of Theorem \ref{thm:reg_CMAB-MT}}\label{apdx_sec:proof_of_CMAB-MT}
By Condition \ref{cond:general_smooth}, we have the weight vector is bounded by $0\le w_{i,j}^{\tilde{\bmu}, \pi}\le \bar{w}$ for all $i\in[m],j\in[d], \tilde{\bmu}\in[0,1]^{m\times d},\pi\in\Pi$. Define the event $\cE_{s,t}$ as the event joint oracle successfully yields an $\alpha$-approximation in round $t\in[T]$, i.e., $\cE_{s,t}=\left\{r(\pi_t,\tilde{\bmu}_t)\ge \alpha \cdot \max_{\pi\in \Pi, \bmu\in \cC(\pi)}r(\pi;\bmu)\right\}$.  Recall that $\cF_{t-1}=((\pi_1, \tau_1, (\bX_{t,i})_{i \in \tau_1}), ..., (\pi_{t-1}, \tau_{t-1}, (\bX_{t,i})_{i \in \tau_{t-1}}), \pi_t)$ is all historical information before $t$ plus the action at $t$. 

We bound the regret under the event $(\cE_{s,t})_{t\in [T]}$, the concentration event $\cE_{c,1}$ in \cref{eq:thm_concen_1} and the concentration event $\cE_{c,2}$ in \cref{eq:thm_concen_2} as follows.

\textbf{Step 1: Regret Decomposition.}

First, recall that $[\bu-\bv]_+\overset{\df}{=}\abs{\bu-\bv}$ for any problem that satisfies the 1-norm MPTM smoothness condition and $[\bu-\bv]_+\overset{\df}{=}(\bu-\bv)$ for problem that satisfies the 1-norm MPTM+ smoothness condition. We can decompose the round $t$ instantaneous regret as follows.
\begin{align}
    \Delta_{\pi_t}&=\alpha \cdot r(\pi^*;\bmu)- r(\pi_t;\bmu)\\
    &\overset{(a)}{\le}r(\pi_t;\tilde{\bmu}_t)-r(\pi_t;\bmu)\\
    &\overset{(b)}{\le} \sum_{i\in[m]}q_i^{\bmu,\pi_t}\abs{\uabs{\tilde{\bmu}_{t,i}-\bmu_{i}}^{\top}\bw_i^{\tilde{\bmu}_t,\pi_t}}\\
    &\le\sum_{i\in[m]:N_{t-1,i}>0}q_i^{\bmu,\pi_t}\abs{\uabs{\tilde{\bmu}_{t,i}-\bmu_{i}}^{\top}\bw_i^{\tilde{\bmu}_t,\pi_t}}+\sum_{i\in[m]: N_{t-1,i}=0}q_i^{\bmu,\pi_t}\bar{w}d\\ &=\sum_{i\in[m]:N_{t-1,i}>0}q_i^{\bmu,\pi_t}\abs{\uabs{\tilde{\bmu}_{t,i}-\hat{\bmu}_{t-1,i}+\hat{\bmu}_{t-1,i}-\bmu_{i}}^{\top}\bw_i^{\tilde{\bmu}_t,\pi_t}}+\sum_{i\in[m]: N_{t-1,i}=0}q_i^{\bmu,\pi_t}\bar{w}d\\
    &\overset{(c1)}{\le}\sum_{i\in[m]:N_{t-1,i}>0}q_i^{\bmu,\pi_t}\abs{\uabs{\tilde{\bmu}_{t,i}-\hat{\bmu}_{t-1,i}}^{\top}\bw_i^{\tilde{\bmu}_t,\pi_t}}+q_i^{\bmu,\pi_t}\abs{\uabs{\hat{\bmu}_{t-1,i}-\bmu_{i}}^{\top}\bw_i^{\tilde{\bmu}_t,\pi_t}}+\sum_{i\in[m]: N_{t-1,i}=0}q_i^{\bmu,\pi_t}\bar{w}d\\
    &\overset{(c2)}{\le}\sum_{i\in[m]:N_{t-1,i}>0}q_i^{\bmu,\pi_t}\abs{\uabs{\tilde{\bmu}_{t,i}-\hat{\bmu}_{t-1,i}}^{\top}\bw_i^{\tilde{\bmu}_t,\pi_t}}+q_i^{\bmu,\pi_t}\abs{\uabs{\bmu_{i}-\hat{\bmu}_{t-1,i}}^{\top}\bw_i^{\bmu,\pi^*}}\notag\\
    &+q_i^{\bmu,\pi_t}\abs{\uabs{\bmu_{i}-\hat{\bmu}_{t-1,i}}^{\top}(\bw_i^{\tilde{\bmu}_t,\pi_t}-\bw_i^{\bmu,\pi^*})}+\sum_{i\in[m]: N_{t-1,i}=0}q_i^{\bmu,\pi_t}\bar{w}d\\
    &\overset{(d)}{\le} \sum_{i\in[m]:N_{t-1,i}>0}q_i^{\bmu,\pi_t}(2F_{t,i}+G_{t,i})\sqrt{\frac{1}{N_{t-1,i}}} + \sum_{i \in [m]:N_{t-1,i}>0}q_i^{\bmu,\pi_t} (2I_{t,i}+J_{t,i})\frac{1}{N_{t-1,i}}+\sum_{i\in[m]: N_{t-1,i}=0}q_i^{\bmu,\pi_t}\bar{w}d\label{apdx_eq:CMAB-MT_analysis_d}\\
    & \overset{(e)}{\le} 2\left(\sqrt{\sum_{i\in[m]:N_{t-1,i}>0}q_i^{\bmu,\pi_t}F_{t,i}^2}+\sqrt{\sum_{i\in[m]:N_{t-1,i}>0}q_i^{\bmu,\pi_t}G_{t,i}^2}\right)\sqrt{\sum_{i\in[m]:N_{t-1,i}>0}\frac{q_i^{\bmu,\pi_t}}{N_{t-1,i}}}\notag\\
    &+ 2(\bar{I}+\bar{J})\sum_{i\in[m]:N_{t-1,i}>0}  \frac{q_i^{\bmu,\pi_t}}{N_{t-1,i}}+\sum_{i\in[m]: N_{t-1,i}=0}q_i^{\bmu,\pi_t}\bar{w}d\label{apdx_eq:CMAB-MT_analysis_e}\\
    & \overset{(f)}{\le} 2\left(\sqrt{\bar{F}}+\sqrt{\bar{G}}\right)\sqrt{\sum_{i\in[m]:N_{t-1,i}>0}\frac{q_i^{\bmu,\pi_t}}{N_{t-1,i}}}+ 2(\bar{I}+\bar{J})\sum_{i\in[m]:N_{t-1,i}>0}  \frac{q_i^{\bmu,\pi_t}}{N_{t-1,i}}+\sum_{i\in[m]: N_{t-1,i}=0}q_i^{\bmu,\pi_t}\bar{w}d\label{apdx_eq:CMAB-MT_analysis_f}
\end{align}
where inequality (a) is due to  $\bmu\in\cC_t(\pi^*)$ and $r(\pi_t,\tilde{\bmu}_t)\ge \alpha \cdot \max_{\pi\in \Pi, \bmu\in \cC(\pi)}r(\pi;\bmu)$ under $\cE_s$, inequality (b) is due to the 1-norm MTPM/MTPM+ smoothness condition (Condition~\ref{cond:general_smooth}), inequality (c1) and (c2) are by triangle inequality for both $[\bu-\bv]_+{=}\abs{\bu-\bv}$ and $[\bu-\bv]_+{=}(\bu-\bv)$ cases, inequality (d) is due to the confidence region \cref{eq:conf_set}, the event $\cE_{c,1}$ in \cref{eq:thm_concen_1} and $\cE_{c,2}$ in \cref{eq:thm_concen_2}, inequality (e) is by Cauchy-Schwarz inequality and the definition of $\bar{I}, \bar{J}$, and inequality (f) is due to the definition of $\bar{F},\bar{G}$.

Let $c_1=3\times 2(\sqrt{\bar{F}}+\sqrt{\bar{G}})$, $c_2=3\times2(\bar{I}+\bar{J})$, $c_3=3\times \bar{w}d$.
Now we define the main event $\cE_t$ and its three decomposed events $\cE'_{t,1},\cE'_{t,2},\cE'_{t,3}$ as follows:
\begin{align}
    \cE_t&=\left\{\Delta_{\pi_t}\le 2(\sqrt{F}+\sqrt{\bar{G}})\sqrt{\sum_{i\in[m]:N_{t-1,i}>0}\frac{q_i^{\bmu,\pi_t}}{N_{t-1,i}}}+2(\bar{I}+\bar{J}) \sum_{i\in[m]:N_{t-1,i}>0}  \frac{q_i^{\bmu,\pi_t}}{N_{t-1,i}}+\sum_{i\in[m]: N_{t-1,i}=0}q_i^{\bmu,\pi_t}\bar{w}d\right\},\\
    \cE'_{t,1}&=\left\{\Delta_{\pi_t}\le c_1\sqrt{\sum_{i\in[m]:N_{t-1,i}>0}\frac{q_i^{\bmu,\pi_t}}{N_{t-1,i}}}\right\},
    \cE'_{t,2}=\left\{\Delta_{\pi_t}\le c_2 \sum_{i\in[m]:N_{t-1,i}>0}  \frac{q_i^{\bmu,\pi_t}}{N_{t-1,i}}\right\},\\
    \cE'_{t,3}&=\left\{\Delta_{\pi_t}\le \sum_{i\in[m]: N_{t-1,i}=0}q_i^{\bmu,\pi_t}c_3\right\}.
\end{align}

By \cref{apdx_lem:decompose_filter_reg}, we have
\begin{align}\label{apdx_eq:regret_decompose_event_filtered}
    \reg(T,\cE_t)\le \sum_{i\in[3]}\reg(T,\cE'_{t,i})
\end{align}

\textbf{Step 2: Bound the $\reg(T,\cE'_{t,1})$ term}

Let $\E_t=[\cdot \mid \cF_{t-1}]$. Suppose $\cE'_{t,1}$ holds, we use the reverse amortization trick as follows:
\begin{align}
     \Delta_{\pi_t} &\overset{(a)}{\le} \sum_{i\in[m]:N_{t-1,i}>0} \frac{ c_1^2  q_{i}^{\bmu,\pi_t}\frac{1}{N_{t-1,i}}}{\Delta_{\pi_t}} \\
    &\overset{(b)}{\le}-\Delta_{\pi_t} + 2\sum_{i\in[m]:N_{t-1,i}>0} \frac{ c_1^2 q_{i}^{\bmu,\pi_t}\frac{1}{N_{t-1,i}}}{\Delta_{\pi_t}}\\
    &\le-\frac{\sum_{i\in[m]:N_{t-1,i}>0}q_i^{\bmu,\pi_t}\Delta_{\pi_t}}{\sum_{i \in [m]}q_i^{\bmu,\pi_t}} + 2\sum_{i\in[m]:N_{t-1,i}>0} \frac{ c_1^2 q_{i}^{\bmu,\pi_t}\frac{1}{N_{t-1,i}}}{\Delta_{\pi_t}}\\
    &\overset{(c)}{\le} \sum_{i\in[m]:N_{t-1,i}>0}q_i^{\bmu,\pi_t} \left(\frac{ 2c_1^2 \frac{1}{N_{t-1,i}}}{\Delta_{\pi_t}}-\frac{\Delta_{\pi_t}}{K}\right)\label{eq:ra_def_e1_1to2}
\end{align}
where inequality (a) follows from event $\cE'_{t,1}$, inequality (b) is due to the reverse amortization trick that multiplies two to both sides of inequality (a) and rearranges the terms, and inequality (c) is due to definition of $K\overset{\df}{=}\max_{\pi \in \Pi}\sum_{i\in[m]}q_i^{\bmu,\pi}$.

Then we use the triggering probability equivalence trick (TPE) in \cite{liu2023contextual} to deal with the triggering probability $q_i^{\bmu,\pi_t}$ as follows: 

\begin{align}
    \E_t[\Delta_{\pi_t}]  &\overset{(a)}{\le}  \E_t\left[\sum_{i\in[m]:N_{t-1,i}>0}q_i^{\bmu,\pi_t}\left(\frac{ 2c_1^2 \frac{1}{N_{t-1,i}}}{\Delta_{\pi_t}}-\frac{\Delta_{\pi_t}}{K}\right)\right]\\
    &\overset{(b)}{=} \E_t\left[ \sum_{i \in \tau_t: N_{t-1,i}>0}\left(\frac{ 2c_1^2 \frac{1}{N_{t-1,i}}}{\Delta_{\pi_t}}-\frac{\Delta_{\pi_t}}{K}\right)\right]
\end{align}
where inequality (a) follows from \Cref{eq:ra_def_e1_1to2}, inequality (b) follows from TPE trick to replace $q_i^{\bmu,\pi_t}=\E_t[\I\{i \in \tau_t\}]$. 

Now we claim that
\begin{align}
    \E_t[\Delta_{\pi_t}]\overset{(a)}{\le} \E_t\left[\sum_{i \in \tau_t: N_{t-1,i}>0}{ \kappa_{i}(N_{t-1,i})}\right],\label{apdx_eq:kappa_tpvm}
\end{align}
where we define $L_{i,1}=\frac{c_1^2}{(\Delta_{i}^{\min})^2}$, $L_{i,2}=\frac{2c_1^2 K}{(\tDelta_i^{\min})^2}$, and
\begin{equation}
\kappa_{i}(\ell) = \begin{cases}
2\sqrt{\frac{c_1^2  }{\ell }}, &\text{if $1\le\ell\le L_{i,1}$,}\\
\frac{2c_1^2}{\Delta_i^{\min} }\frac{1}{\ell}, &\text{if $L_{i,1} < \ell \le L_{i,2}$,}\\
0, &\text{if $\ell > L_{i,2}$,}
\end{cases}
\end{equation}

We now show inequality (a) because of the following argument.

Case 1: If there exists $i'\in\tau_t$ with $1\le N_{t-1,i'}\le \frac{c_1^2}{\Delta_{\pi_t}^2}$. 

We have $N_{t-1,i'}\le \frac{c_1^2}{\Delta_{\pi_t}^2}\le L_{i',1}$, thus $\sum_{i\in\tau_t:N_{t-1,i}>0}\kappa_i(N_{t-1,i})\ge \kappa_{i'}(N_{t-1,i'}) \ge2\sqrt{\frac{c_1^2}{N_{t-1,i'}}}= 2\Delta_{\pi_t}$, then inequality (a) holds. 

Case 2: For any arm $i\in\tau_t$ with $N_{t-1,i}>0$, they satisfy $N_{t-1,i}\ge\frac{c_1^2}{\Delta_{\pi_t}^2}$. 

If $N_{t-1,i}\le L_{i,1}$, then $ \frac{ 2c_1^2 \frac{1}{N_{t-1,i}}}{\Delta_{\pi_t}}-\frac{\Delta_{\pi_t}}{K}\le \frac{ 2c_1^2 \frac{1}{N_{t-1,i}}}{\Delta_{\pi_t}}=2\sqrt{\frac{c_1^2}{\Delta_{\pi_t}^2\cdot N_{t-1,i}}}\sqrt{\frac{c_1^2}{N_{t-1,i}}}\le 2\sqrt{\frac{c_1^2}{N_{t-1,i}}}=\kappa_i(N_{t-1,i})$; Else if $L_{i,1}<N_{t-1,i}\le L_{i,2}$, then $\frac{ 2c_1^2 \frac{1}{N_{t-1,i}}}{\Delta_{\pi_t}}-\frac{\Delta_{\pi_t}}{K}\le \frac{ 2c_1^2 \frac{1}{N_{t-1,i}}}{\Delta_{\pi_t}}\le \frac{ 2c_1^2 \frac{1}{N_{t-1,i}}}{\Delta_{i}^{\min}}=\kappa_i(N_{t-1,i})$; Else if $N_{t-1,i}>L_{i,2}$, $ \frac{ 2c_1^2 \frac{1}{N_{t-1,i}}}{\Delta_{\pi_t}}-\frac{\Delta_{\pi_t}}{K}\le 0=\kappa_i(N_{t-1,i})$. Therefore, we have 
\begin{align}
    \E_t[\Delta_{\pi_t}]\le \E_t\left[ \sum_{i \in \tau_t: N_{t-1,i}>0}\left(\frac{ 2c_1^2 \frac{1}{N_{t-1,i}}}{\Delta_{\pi_t}}-\frac{\Delta_{\pi_t}}{K}\right)\right]\le \E\left[\sum_{i\in\tau_t:N_{t-1,i}>0}\kappa_i(N_{t-1,i})\right].
\end{align}
Combining the above two cases proves inequality (a).


Now we have
\begin{align}
    \reg(T, \cE'_{t,1})&= \E\left[\sum_{t=1}^T\Delta_{\pi_t}\right]\\
    &\overset{(a)}{\le} \E\left[\sum_{t\in [T]}\E_t\left[\sum_{i \in \tau_t: N_{t-1,i}>0}{ \kappa_{i}(N_{t-1,i})}\right]\right]\\
    &\overset{(b)}{=} \E\left[\sum_{t\in [T]}\sum_{i \in \tau_t: N_{t-1,i}>0}{ \kappa_{i}(N_{t-1,i})}\right]\\
    &\overset{(c)}{=}\E\left[\sum_{i \in [m]}\sum_{s=1}^{N_{T-1,i}}\kappa_{i}(s)\right]\\
    &\le  \sum_{i \in [m]} \sum_{s=1}^{L_{i,1}}2\sqrt{\frac{c_1^2 }{s}} + \sum_{i \in [m]}\sum_{s=L_{i,1}+1}^{L_{i,2}}\frac{2c_1^2}{\Delta_i^{\min} }\frac{1}{s}\\
    &\le  \sum_{i \in [m]} \int_{s=0}^{L_{i,1}}2\sqrt{\frac{c_1^2 }{s}}\cdot ds + \sum_{i \in [m]}\int_{s=L_{i,1}}^{L_{i,2}}\frac{2c_1^2}{\Delta_i^{\min} }\frac{1}{s}\cdot ds\\
    &\le \sum_{i \in [m]} \frac{2c_1^2}{\Delta_{i}^{\min} } (3+ \log K),\label{apdx_eq:CMAB-MT_term1}
\end{align}
where (a) follows from \Cref{apdx_eq:kappa_tpvm}, (b) follows from the tower rule, (c) follows from that $N_{t-1,i}$ is increased by $1$ if and only if $i \in \tau_t$.

\textbf{Step 3: Bound the $\reg(T,\cE'_{t,2})$ term}

Let $\E_t=[\cdot \mid \cF_{t-1}]$. Suppose $\cE'_{t,2}$ holds, we use the reverse amortization trick as follows:
\begin{align}
     \Delta_{\pi_t} &\overset{(a)}{\le} \sum_{i \in [m]: N_{t-1,i}>0} c_2 q_{i}^{\bmu,\pi_t} \frac{1}{N_{t-1,i}}\notag\\
    &\overset{(b)}{\le} -\Delta_{\pi_t} + 2\sum_{i \in [m]: N_{t-1,i}>0} c_2 q_{i}^{\bmu,\pi_t} \frac{1}{N_{t-1,i}} \notag\\
    &=-\frac{\sum_{i \in [m]: N_{t-1,i}>0}q_i^{\bmu, \pi_t}\Delta_{\pi_t}}{\sum_{i \in [m]: N_{t-1,i}>0}q_i^{\bmu, \pi_t}}+ 2\sum_{i \in [m]: N_{t-1,i}>0} c_2 q_{i}^{\bmu,\pi_t} \frac{1}{N_{t-1,i}} \\
     &\overset{(c)}{\le}   \sum_{\sum_{i \in [m]: N_{t-1,i}>0}}q_i^{\bmu, \pi_t}\left(-\frac{\Delta_{\pi_t}}{K}+ 2c_2\frac{1}{N_{t-1,i}}\right),\label{eq:ra_inf_prob_e2}
\end{align}
where inequality (a) follows from event $E_{t,2}$, inequality (b) is due to the reverse amortization trick that multiplies two to both sides of inequality (a) and rearranges the terms, inequality (c) is due to definition of $K\overset{\df}{=}\max_{\pi \in \Pi}\sum_{i\in[m]}q_i^{\bmu,\pi}$.

It follows that
\begin{align}
    \E_t[\Delta_{\pi_t}]  &\overset{(a)}{\le}  \E_t\left[\sum_{i \in [m]}q_i^{\bmu, \pi_t}\left(-\frac{\Delta_{\pi_t}}{K}+ 2c_2 \frac{1}{N_{t-1,i}}\right)\right]\notag\\
    &\overset{(b)}{=} \E_t\left[ \sum_{i \in \tau_t: N_{t-1,i}>0}\left(-\frac{\Delta_{\pi_t}}{K}+ 2c_2 \frac{1}{N_{t-1,i}}\right)\right]\notag\\
    &\overset{(c)}{\le} \E_t\left[\sum_{i \in \tau_t: N_{t-1,i}>0}{ \kappa_{i}(N_{t-1,i})}\right]\label{apdx_eq:tpvm_lambda_term_2}
\end{align}
where the following regret allocation function follows from
\begin{equation}
\kappa_{i}(\ell) = \begin{cases}
\frac{2c_2}{\ell}, &\text{if $1 \le \ell \le L_{i}$}\\
0, &\text{if $\ell > L_{i} + 1$,}
\end{cases}
\end{equation}
where $L_{i}=\frac{2c_2 K}{\Delta_i^{\min}}$.
And inequality (a) follows from \Cref{eq:ra_inf_prob_e2}, (b) is due to the TPE to replace $q_i^{\bmu,\pi_t}=\E_t[\I\{i \in \tau_t\}]$, (c) follows from the fact that if $N_{t-1,i}>\frac{2c_2 K}{\Delta_i^{\min}}$, then $-\frac{\Delta_{\pi_t}}{K}+ 2c_2 \frac{1}{N_{t-1,i}}\le -\frac{\Delta_{\pi_t}}{K}+ \frac{\Delta_i^{\min}}{K}\le 0$; Else, $-\frac{\Delta_{\pi_t}}{K}+ 2c_2 \frac{1}{N_{t-1,i}}\le 2c_2 \frac{1}{N_{t-1,i}}$.

\begin{align}
    Reg(T, E_{t,2})&= \E\left[\sum_{t=1}^T\Delta_{\pi_t}\right]\\
    &\overset{(a)}{\le} \E\left[\sum_{t\in [T]}\E_t\left[\sum_{i \in \tau_t: N_{t-1,i}>0}{ \kappa_{i}(N_{t-1,i})}\right]\right]\\
    &\overset{(b)}{=} \E\left[\sum_{t\in [T]}\sum_{i \in \tau_t: N_{t-1,i}>0}{ \kappa_{i}(N_{t-1,i})}\right]\\
    &\overset{(c)}{=}\E\left[\sum_{i \in [m]}\sum_{s=1}^{N_{T-1,i}}\kappa_{i}(s)\right]\\
     &\le\sum_{i\in [m]} \sum_{\ell=1}^{L_{i}}\frac{2c_2}{\ell}\\
    &\le \sum_{i\in[m]}2c_2 \left(1+ \int_{s=1}^{L_{i}}\frac{1}{s}\cdot ds\right)\\
    &= \sum_{i\in [m]} 2c_2  \left(1+\log\left (\frac{2c_2K}{\Delta_{i}^{\min}}\right)\right),\label{apdx_eq:CMAB-MT_term2}
\end{align}
where (a) follows from \Cref{apdx_eq:tpvm_lambda_term_2}, (b) follows from the tower rule, (c) follows from that $N_{t-1,i}$ is increased by $1$ if and only if $i \in \tau_t$.

\textbf{Step 4: Bound the $\reg(T,\cE'_{t,3})$ term}

By \cref{apdx_lem:null_counter}, we have 
\begin{align}\label{apdx_eq:CMAB-MT_term3}
    \reg(T,\cE'_{t,3})\le c_3m.
\end{align}

\textbf{Step 5: Putting everything together}

Plugging \cref{apdx_eq:CMAB-MT_term1}, \cref{apdx_eq:CMAB-MT_term2}, \cref{apdx_eq:CMAB-MT_term3} into \cref{apdx_eq:regret_decompose_event_filtered}, we have 
\begin{align}
    \reg(T,\cE_t)&\le \sum_{i \in [m]} \frac{2c_1^2}{\Delta_{i}^{\min} } (3+ \log K) + \sum_{i\in [m]} 2c_2  \left(1+\log\left (\frac{2c_2K}{\Delta_{i}^{\min}}\right)\right) + c_3m\\
    &\le\sum_{i \in [m]} \frac{144(\bar{F}+\bar{G})}{\Delta_{i}^{\min} } (3+ \log K) + \sum_{i\in [m]} 12(\bar{I}+\bar{J})  \left(1+\log\left (\frac{12(\bar{I}+\bar{J})K}{\Delta_{i}^{\min}}\right)\right) + 3m\bar{w}d\\
    &=O\left(\sum_{i \in [m]} \frac{(\bar{F}+\bar{G})}{\Delta_{i}^{\min} } + 2(\bar{I}+\bar{J})  \log\left (\frac{2(\bar{I}+\bar{J})K}{\Delta_{i}^{\min}}\right)\right)\label{apdx_eq:CMAB-MT_gap_dep_reg}
\end{align}

Let $\reg(T,\{\})\overset{\df}{=} \E\left[\sum_{t\in[T]}(\alpha\cdot r(\pi^*;\bmu)-r(\pi_t;\bmu))\right]$ be the regret event without any filter events. 
Now consider the regret caused by the failure event $(\neg \cE_{s,t})_{t\in[T]}, \neg \cE_{c,1}, \neg\cE_{c,2}$, we have \begin{align}
    \reg(T,\{\})&\le \reg(T, \cE_{s,t}\cap\cE_{c,1}\cap \cE_{c,2})+ \reg(T,\neg\cE_{s,t})+\reg(T,\neg\cE_{c,1})+\reg(T,\neg\cE_{c,2})\\
    &   \overset{(a)}{\le} \reg(T,\cE_t) + (1-\beta)T\Delta_{\max}+2\Delta_{\max}
\end{align}
where inequality (a) is due to $\cE_{s,t}\cap\cE_{c,1}\cap \cE_{c,2}$ implies $\cE_t$ for the first term and the fact that $\Delta_{\pi_t}\le \Delta_{\max}, \Pr[\neg \cE_{s,t}]\le \beta, \Pr[\neg \cE_{c,1}]\le 1/T, \Pr[\neg \cE_{c,2}]\le 1/T $ for the rest of the terms.

Therefore we can derive that the regret is upper bounded by 
\begin{align}
\reg(T;\alpha,\beta,\bmu)&=\reg(T,\{\})-(1-\beta)T\alpha r(\pi^*;\bmu)\\
&\le \reg(T,\{\})-(1-\beta)T\Delta_{\max}\\
&\le \reg(T,\cE_t)+2\Delta_{\max}\\
&=O\left(\sum_{i \in [m]} \frac{(\bar{F}+\bar{G})}{\Delta_{i}^{\min} } + 2(\bar{I}+\bar{J})  \log\left (\frac{2(\bar{I}+\bar{J})K}{\Delta_{i}^{\min}}\right)\right)\label{eq:similar_total}
\end{align}

For the gap-independent regret take $\Delta=\sqrt{144m(\bar{F}+\bar{G})/T}$. On one hand, $\reg(T, \I\{\Delta_{\pi_t}< \Delta\}\cap \cE_t)\le T\Delta$; On the other hand, $\reg(T,\I\{\Delta_{\pi_t}\ge \Delta\}\cap \cE_t)\le \sum_{i \in [m]} \frac{144(\bar{F}+\bar{G})}{\Delta} (3+ \log K) + \sum_{i\in [m]} 12(\bar{I}+\bar{J})  \left(1+\log\left (\frac{12(\bar{I}+\bar{J})K}{\Delta}\right)\right) + 3m\bar{w}d$ according to \cref{apdx_eq:CMAB-MT_gap_dep_reg}.
Therefore, $\reg(T;\alpha,\beta,\bmu)\le \reg(T,\I\{\Delta_{\pi_t} < \Delta\}\cap \cE_t)+\reg(T,\I\{\Delta_{\pi_t}\ge \Delta\}\cap \cE_t)+2\Delta_{\max}\le O\left(\sqrt{(\bar{F}+\bar{G})mT}+m(\bar{I}+\bar{J})\log (KT)\right)$ using the similar proof of \cref{eq:similar_total}.

\section{Analysis for Episodic RL with Sublinear Regret in Section~\ref{sec:RL_sec_loose}}

\subsection{Proof of Lemma~\ref{lem:RL_smooth_loose} and Lemma~\ref{lem:RL_tight_smooth}}\label{apdx_sec:proof_lem1n2}

We use $q^{\bp,\pi}_{(s',i)}(s,a,h)$ to denote the triggering probability $q_{s,a,h}^{\bp, \pi}$ when the policy is $\pi$, the transition is $\bp$, and starting from initial state $s'$ at step $i$. For notational simplicity, we use $q_{(s',i)}(s,a,h)$ to denote $q^{\bp,\pi}_{(s',i)}(s,a,h)$ and $\tilde{q}_{(s',i)}(s,a,h)$ to denote $q^{\tilde{\bp},\pi}_{(s',i)}(s,a,h)$. Similarly, we use $q(s,a,h)$ to denote $q^{\bp,\pi}_{(s_1,1)}(s,a,h)$ and $\tilde{q}(s,a,h)$ to denote $q^{\tilde{\bp},\pi}_{(s_1,1)}(s,a,h)$ when the starting state is fixed from $s_1$ at the initial step.

\begin{lemma}[Smoothness of the Triggering Probability and the Value Function]\label{apdx_lem:trigger_smooth}
    For any triggering probability $q$ and $\tilde{q}$ given by the same policy $\pi$ but different transition $\bp$ and $\tilde{\bp}$, respectively, we have
    \begin{align}\label{apdx_eq:trigger_smoothness}
        \tilde{q}(s,a,h)-q(s,a,h)= \sum_{(s',a'),s''}\sum_{i=1}^{h-1}q(s',a',i)(\tilde{p}(s''|s',a',i)-p(s''|s',a',i))\tilde{q}_{(s'',i+1)}(s,a,h).
    \end{align}
    And we have
    \begin{align}\label{apdx_eq:value_smoothness}
        \abs{V_1^{\tilde{\bp},\pi}(s_1)-V_1^{\bp,\pi}(s_1)}&\le \sum_{s,a,h}q_{s,a,h}^{\bp, \pi} \abs{[\tilde{\bp}(s,a,h)-\bp(s,a,h)]^{\top} \bV_{h+1}^{\tilde{\bp},\pi}}\\
        &\le H\sum_{s,a,h}q_{s,a,h}^{\bp, \pi} \norm{\tilde{\bp}(s,a,h)-\bp(s,a,h)}_1
    \end{align}
\end{lemma}
\begin{proof}
    We first prove \cref{apdx_eq:trigger_smoothness} by induction on $h$. When $h=1$, then $\tilde{q}(s,a,h)=q(s,a,h)=\I\{\pi(s,h)=a\}\I\{s=s_1\}$, \cref{apdx_eq:trigger_smoothness} holds. For the induction step $h>1$:
    \begin{align}
        &\tilde{q}(s,a,h)-q(s,a,h)\overset{(a)}{=}\I\{\pi(s,h)=a\}\left[\sum_{s',a'}\tilde{q}(s',a',h-1)\tilde{p}(s|s',a',h-1)-q(s',a',h-1)p(s|s',a',h-1)\right]\\
        &=\underbrace{\I\{\pi(s,h)=a\}\sum_{s',a'}\tilde{p}(s|s',a',h-1)(\tilde{q}(s',a',h-1)-q(s',a',h-1))}_{\text{Term 1}}\notag\\
        &+\underbrace{\I\{\pi(s,h)=a\}\sum_{s',a'}q(s',a',h-1)(\tilde{p}(s|s',a',h-1)-p(s|s',a',h-1))}_{\text{Term 2}}
    \end{align}
    where inequality (a) is due to $q(s,a,h)=\sum_{s',a'}\I\{\pi(s,h)=a\}q(s',a',h-1)p(s|s',a',h-1)$.
    Then we bound Term 1 and Term 2 as follows.
    \begin{align}
        \text{Term 1}&\overset{(a)}{=}\I\{\pi(s,h)=a\}\sum_{s',a'}\tilde{p}(s|s',a',h-1)\notag\\
        &\cdot\left(\sum_{(s'',a''),s'''}\sum_{i=1}^{h-2}q(s'',a'',i)(\tilde{p}(s'''|s'',a'',i)-p(s'''|s'',a'',i))\tilde{q}_{(s''',i+1)}(s',a',h-1)\right)\\
        &=\sum_{(s'',a''),s'''}\sum_{i=1}^{h-2}q(s'',a'',i)(\tilde{p}(s'''|s'',a'',i)-p(s'''|s'',a'',i))\notag\\
        &\cdot\sum_{s',a'}\tilde{p}(s|s',a',h-1)\I\{\pi(s,h)=a\}\tilde{q}_{(s''',i+1)}(s',a',h-1)\\
        &\overset{(b)}{=}\sum_{(s'',a''),s'''}\sum_{i=1}^{h-2}q(s'',a'',i)(\tilde{p}(s'''|s'',a'',i)-p(s'''|s'',a'',i))\tilde{q}_{(s''',i+1)}(s,a,h)\\
    \end{align}
    where equality (a) is due to the induction hypothesis, equality (b) is due to $\tilde{q}_{(s''',i+1)}(s,a,h)=\sum_{s',a'}\tilde{p}(s|s',a',h-1)\I\{\pi(s,h)=a\}\tilde{q}_{(s''',i+1)}(s',a',h-1)$.

\begin{align}
    \text{Term 2}&=\sum_{(s',a'),s''}\I\{\pi(s'',h)=a\}\I\{s''=s\}q(s',a',h-1)(\tilde{p}(s''|s',a',i)-p(s''|s',a',i))\\
    &\overset{(a)}{=}\sum_{(s',a'),s''}q(s',a',h-1)(\tilde{p}(s''|s',a',i)-p(s''|s',a',i))\tilde{q}_{(s'',h)}(s,a,h)
\end{align}
where equality (a) is due to $\tilde{q}_{(s'',h)}(s,a,h)=\I\{\pi(s'',h)=a\}\I\{s''=s\}$.

Plugging in Term 1 (with changing variables $s''',s'',a''$ to $s'', s', a'$) and Term 2 proves the \cref{apdx_eq:trigger_smoothness}. 

Now we prove the smoothness for the value function for the second part of \cref{apdx_lem:trigger_smooth}.
\begin{align}
    \abs{V_1^{\tilde{\bp},\pi}(s_1)-V_1^{\bp,\pi}(s_1)}&{=} \abs{\sum_{s,a,h}(\tilde{q}(s,a,h)-q(s,a,h))r(s,a,h)}\\
    &\overset{(a)}{=}\abs{\sum_{s,a,h}\sum_{(s',a'),s''}\sum_{i=1}^{h-1}q(s',a',i)(\tilde{p}(s''|s',a',i)-p(s''|s',a',i))\tilde{q}_{(s'',i+1)}(s,a,h)r(s,a,h)}\\
    &{=}\abs{\sum_{i=1}^{H}\sum_{(s',a'),s''}q(s',a',i)(\tilde{p}(s''|s',a',i)-p(s''|s',a',i))\sum_{s,a,h>i}\tilde{q}_{(s'',i+1)}(s,a,h)r(s,a,h)}\\
    &{=}\abs{\sum_{i=1}^{H}\sum_{(s',a'),s''}q(s',a',i)(\tilde{p}(s''|s',a',i)-p(s''|s',a',i))V^{\tilde{p},\pi}_{h+1}(s'')}\\
    &\le \sum_{s,a,h}q_{s,a,h}^{\bp, \pi} \abs{[\tilde{\bp}(s,a,h)-\bp(s,a,h)]^{\top} \bV_{h+1}^{\tilde{\bp},\pi}}\\
        &\overset{(b)}{\le} H\sum_{s,a,h}q_{s,a,h}^{\bp, \pi} \norm{\tilde{\bp}(s,a,h)-\bp(s,a,h)}_1
\end{align}
where equality (a) is due to the first part of \cref{apdx_lem:trigger_smooth} we just proved, inequality (b) is due to $\bV_{h+1}^{\tilde{\bp},\pi}(s)\le H$.
\end{proof}

\subsection{Proof of Theorem \ref{thm:reg_RL_loose}}\label{apdx_sec:RL_loose_proof}
We suppose the concentration event $\cE_{\text{tran1}}$ as defined in \cref{apdx_eq:event_trans1} holds with probability $\delta'=1/(2T)$. Let $L=\log(SAHT)$. Since $\bmu\in \cC_t$ as defined in \cref{eq:RL_loose_conf_set} due to the event $\cE_{\text{tran1}}$, we follow the same regret decomposition and derivation as in Step 1 of \cref{apdx_sec:proof_of_CMAB-MT}, and identify $F_{t,s,a,h}=2H\sqrt{SL}, I_{t,s,a,h}=0$ due to the definition of $\phi_t(s,a,h)$ in \cref{eq:RL_loose_conf_set}, and $G_{t,s,a,h}=J_{t,s,a,h}=0$ since $\bw_{s,a,h}^{\bp,\pi}=H\cdot \boldsymbol{1}$ are constants. Now we can derive that $\bar{F}=\sum_{s,a,h}q_{s,a,h}^{\bp,\pi}F_{t,s,a,h}^2=4H^3SL, \bar{G}=\bar{I}=\bar{J}=0,\bar{w}=H, d=S$, we have
\begin{align}
    \Delta_{\pi_t} & {\le} 2\left(\sqrt{4H^3SL}\right)\sqrt{\sum_{i\in[m]:N_{t-1,i}>0}\frac{q_i^{\bp,\pi_t}}{N_{t-1,i}}}+\sum_{i\in[m]: N_{t-1,i}=0}q_i^{\bp,\pi_t}SH
\end{align}

Following step 2,3 in \cref{apdx_sec:proof_of_CMAB-MT}, we have gap-dependent regret
\begin{align}
Reg(T)=O\left(\sum_{s,a,h}\frac{H^3S\log (SAHT)}{\Delta_{s,a,h}^{\min}}\right)
\end{align}
and gap-independent regret
\begin{align}
Reg(T)=O\left(\sqrt{H^4S^2AT\log (SAHT)}\right).
\end{align}

\section{Analysis that Achieve Minimax Optimal Regret for Episodic RL in Section~\ref{sec:tight_RL}}\label{sec:unknown_trans}

\subsection{Concentration Inequalities}

\begin{lemma}[Concentration of the Transition]\label{apdx_lem:concen_tran}
\begin{align}
    \Pr \Big[&\norm{\hat{\bp}_{t-1}(s,a,h)-\bp(s,a,h)}_1\le \sqrt{\frac{2S\log\left(\frac{SAHT}{\delta'}\right)}{N_{t-1}(s,a,h)}},
    \text{ for any }(s,a,h)\in \cS\times \cA\times [H], t\in [T]\Big]\ge 1-2\delta'
\end{align}
and 
\begin{align}
    \Pr \Big[&\abs{\hat{p}_{t-1}(s'|s,a,h)-p(s'|s,a,h)}\le \sqrt{\frac{p(s'|s,a,h)(1-p(s'|s,a,h))\log\left(\frac{SAHT}{\delta'}\right)}{N_{t-1}(s,a,h)}}+\frac{\logL}{N_{t-1}(s,a,h)},\notag\\
    &\text{ for any }(s,a,h)\in \cS\times \cA\times [H], t\in [T]\Big]\ge 1-2\delta'
\end{align}

\end{lemma}
\begin{proof}
    Using the \cite{weissman2003inequalities} for the first part, and using the Bernstein inequality for the second part, and taking the union bound over $s,a,h,t\in \cS\times \cA\times [H]\times [T]$ and the counter $N_{t-1}(s,a,h)\in[T]$, we obtain the lemma.
\end{proof}

\begin{lemma}[Concentration of the Optimal Future Value]\label{apdx_lem:concen_future}
\begin{align}
    \Pr \Big[&\abs{\left(\hat{\bp}_{t-1}(s,a,h)-\bp(s,a,h)\right)^{\top}\bV^*_{h+1}}\le 2\sqrt{\frac{\Var_{s'\sim \bp(s,a,h)}\left(V_{h+1}^*(s')\right)\log\left(\frac{SAHT}{\delta'}\right)}{N_{t-1}(s,a,h)}}+\frac{H\logL}{N_{t-1}(s,a,h)},\notag\\
    &\text{for any }(s,a,h)\in \cS\times \cA\times [H], t\in [T]\Big]\ge 1-2\delta'
\end{align}
\begin{proof}
    Using the Bernstein inequality and taking the union bound over $s,a,h,t\in \cS\times \cA\times [H]\times [T]$ and the counter $N_{t-1}(s,a,h)\in[T]$, we obtain the lemma.
\end{proof}
\end{lemma}

\begin{lemma}[Concentration of the Variance]\label{apdx_lem:concen_var}
\begin{align}
    \Pr \Big[&\abs{\sqrt{\Var_{s'\sim \hat{\bp}_{t-1}(s,a,h)}\left(\bar{V}_{t,h+1}(s')\right)}-\sqrt{\Var_{s'\sim \bp(s,a,h)}\left(V^*_{h+1}(s')\right)}}\le \sqrt{\E_{s'\sim \hat{\bp}_{t-1}(s,a,h)}\left[\bar{V}_{t,h+1}(s')-V^*_{h+1}(s')\right]^2}\notag\\
    &+2H\sqrt{\frac{\logL}{N_{t-1}(s,a,h)}}, \text{ for any }(s,a,h)\in \cS\times \cA\times [H], t\in [T]\Big]\ge 1-2\delta'
\end{align}
\end{lemma}
\begin{proof}
    According to Proposition 2 (i.e., Eq. (53)) in \cite{zanette2019tighter}, we have 
    \begin{align}\label{apdx_eq:z_variance}
        \Pr \Big[&\abs{\sqrt{\Var_{s'\sim \hat{\bp}_{t-1}(s,a,h)}\left(V_{h+1}^*(s')\right)}-\sqrt{\Var_{s'\sim \bp(s,a,h)}\left(V^*_{h+1}(s')\right)}}\le 
    2H\sqrt{\frac{\logL}{N_{t-1}(s,a,h)}},\notag\\& \text{ for any }(s,a,h)\in \cS\times \cA\times [H], t\in [T]\Big]\ge 1-2\delta'
    \end{align}
    Then we have
    \begin{align}
        &\abs{\sqrt{\Var_{s'\sim \hat{\bp}_{t-1}(s,a,h)}\left(\bar{V}_{t,h+1}(s')\right)}-\sqrt{\Var_{s'\sim \bp(s,a,h)}\left(V^*_{h+1}(s')\right)}}\\
        &\le \abs{\sqrt{\Var_{s'\sim \hat{\bp}_{t-1}(s,a,h)}\left(\bar{V}_{t,h+1}(s')\right)}-\sqrt{\Var_{s'\sim \hat{\bp}_{t-1}(s,a,h)}\left(V^*_{h+1}(s')\right)}}\notag\\
        &+\abs{\sqrt{\Var_{s'\sim \hat{\bp}_{t-1}(s,a,h)}\left(V^*_{h+1}(s')\right)}-\sqrt{\Var_{s'\sim \bp(s,a,h)}\left(V^*_{h+1}(s')\right)}}\\
        &\overset{(a)}{\le} \sqrt{\E_{s'\sim \hat{\bp}_{t-1}(s,a,h)}\left[\bar{V}_{t,h+1}(s')-V^*_{h+1}(s')\right]^2} + 2H\sqrt{\frac{\logL}{N_{t-1}(s,a,h)}}
    \end{align}
    where inequality (a) is due to Eq. (48)-(52) in \cite{zanette2019tighter} and \cref{apdx_eq:z_variance} that holds with probability at least $1-2\delta'$.
\end{proof}

Based on the concentration lemmas we use, we define the following events.

\begin{align}
\cE_{\text{tran1}}\overset{\df}{=} \Big[&\norm{\hat{\bp}_{t-1}(s,a,h)-\bp(s,a,h)}_1\le \sqrt{\frac{2S\log\left(\frac{SAHT}{\delta'}\right)}{N_{t-1}(s,a,h)}},\text{ for any }(s,a,h)\in \cS\times \cA\times [H], t\in [T]\Big]\label{apdx_eq:event_trans1}\\
\cE_{\text{tran2}}\overset{\df}{=} \Big[&\abs{\hat{p}_{t-1}(s'|s,a,h)-p(s'|s,a,h)}\le \sqrt{\frac{p(s'|s,a,h)(1-p(s'|s,a,h))\log\left(\frac{SAHT}{\delta'}\right)}{N_{t-1}(s,a,h)}}+\frac{\logL}{N_{t-1}(s,a,h)}\Big]\\
   \cE_{\text{future}}\overset{\df}{=}\Big[&\abs{\left(\hat{\bp}_{t-1}(s,a,h)-\bp(s,a,h)\right)^{\top}\bV^*_{h+1}}\le 2\sqrt{\frac{\Var_{s'\sim \bp(s,a,h)}\left(V_{h+1}^*(s')\right)\log\left(\frac{SAHT}{\delta'}\right)}{N_{t-1}(s,a,h)}}+\frac{H\logL}{N_{t-1}(s,a,h)},\notag\\
    &\text{for any }(s,a,h)\in \cS\times \cA\times [H], t\in [T]\Big]\\
    \cE_{\text{var}}\overset{\df}{=}\Big[&\abs{\sqrt{\Var_{s'\sim \hat{\bp}_{t-1}(s,a,h)}\left(\bar{V}_{t,h+1}(s')\right)}-\sqrt{\Var_{s'\sim \bp(s,a,h)}\left(V^*_{h+1}(s')\right)}}\le \sqrt{\E_{s'\sim \hat{\bp}_{t-1}(s,a,h)}\left[\bar{V}_{t,h+1}(s')-V^*_{h+1}(s')\right]^2}\notag\\
    &+2H\sqrt{\frac{\logL}{N_{t-1}(s,a,h)}}, \text{ for any }(s,a,h)\in \cS\times \cA\times [H], t\in [T]\Big]\\
\end{align}

\begin{align}
\cE=\cE_{\text{tran1}}\cap\cE_{\text{tran2}}\cap\cE_{\text{future}}\cap\cE_{\text{var}}
\end{align}\label{apdx_eq:concen_events}

\begin{lemma}[High Probability Event]
Let $\delta=\delta'/8$, then
    \begin{align}
        \Pr[\cE]\ge 1-8\delta'=1-\delta.
    \end{align}
\end{lemma}
\begin{proof}
    We can obtain this lemma by \cref{apdx_lem:concen_tran,apdx_lem:concen_future,apdx_lem:concen_var}.
\end{proof}

\begin{lemma}[Concentration of the Optimal Future Value Regarding Known Statistics]\label{apdx_lem:concen_known_future}
Let $L=\log\left(\frac{8SAHT}{\delta}\right)$. Let $\phi_{t}(s,a,h)=2\sqrt{\frac{\Var_{s'\sim \hat{\bp}_{t-1}(s,a,h)}\left(\bar{V}_{t,h+1}(s')\right)L}{N_{t-1}(s,a,h)}}+2\sqrt{\frac{\E_{s'\sim \hat{\bp}_{t-1}(s,a,h)}\left[\bar{V}_{t,h+1}(s')-\ubar{V}_{t,h+1}(s')\right]^2L}{N_{t-1}(s,a,h)}}+\frac{5HL}{N_{t-1}(s,a,h)}$.
With probability at least $1-\delta$, we have 
\begin{align}
    \abs{\left(\hat{\bp}_{t-1}(s,a,h)-\bp(s,a,h)\right)^{\top}\bV^*_{h+1}}\le \phi_t(s,a,h)
\end{align}
\end{lemma}
\begin{proof}
    Under $\cE$, we can obtain the lemma by applying \cref{apdx_lem:concen_future}, \cref{apdx_lem:concen_var}, and \cref{apdx_lem:onp_pointwise}
\end{proof}

\subsection{Optimism and Pessimism}
Let $L=\logL$. Let $\phi_{t}(s,a,h)=2\sqrt{\frac{\Var_{s'\sim \hat{\bp}_{t-1}(s,a,h)}\left(\bar{V}_{t,h+1}(s')\right)L}{N_{t-1}(s,a,h)}}+2\sqrt{\frac{\E_{s'\sim \hat{\bp}_{t-1}(s,a,h)}\left[\bar{V}_{t,h+1}(s')-\ubar{V}_{t,h+1}(s')\right]^2L}{N_{t-1}(s,a,h)}}+\frac{5HL}{N_{t-1}(s,a,h)}$.

\begin{lemma}\label{apdx_lem:onp_pointwise}
    If concentration event $\cE$ holds, then it holds that
    \begin{align}
        \ubar{V}_{t,h}(s)\le V^*_{h}(s)\le \bar{V}_{t,h}(s) 
    \end{align}
    for all $s\in\cS, h\in[H],t\in[T]$.
\end{lemma}
\begin{proof}
    We prove this lemma by induction. Since it holds that $\ubar{V}_{t,H+1}(s)= V^*_{H+1}(s)= \bar{V}_{t,H+1}(s)=0$, so it suffices to prove that if $\ubar{V}_{t,h+1}(s)\le V^*_{h+1}(s)\le \bar{V}_{t,h+1}(s) $, then $\ubar{V}_{t,h}(s)\le V^*_{h}(s)\le \bar{V}_{t,h}(s)$.

    We first prove the optimistic part, i.e., $V^*_{h}(s)\le \bar{V}_{t,h}(s)$. If 
    \begin{align}
        r(s,\pi_t(s,h),h)+\hat{\bp}_{t-1}(s,\pi_t(s,h),h)^{\top} \bar{\bV}_{t,h+1} + \phi_t(s,\pi_t(s,h),h)\ge H-h, 
    \end{align}
     then we are done. If the above does not hold, we have the following,
     \begin{align}
         \bar{V}_{t,h}(s)&=r(s,\pi_t(s,h),h)+\hat{\bp}_{t-1}(s,\pi_t(s,h),h)^{\top} \bar{\bV}_{t,h+1} + \phi_t(s,\pi_t(s,h),h)\\
         & \overset{(a)}{\ge} r(s,\pi^*(s,h),h)+\hat{\bp}_{t-1}(s,\pi^*(s,h),h)^{\top} \bar{\bV}_{t,h+1} + \phi_t(s,\pi^*(s,h),h)\\
         & \overset{(b)}{\ge} r(s,\pi^*(s,h),h)+\hat{\bp}_{t-1}(s,\pi^*(s,h),h)^{\top} \bV^*_{h+1} + \phi_t(s,\pi^*(s,h),h)\\
         & \overset{(c)}{\ge} r(s,\pi^*(s,h),h)+\hat{\bp}_{t-1}(s,\pi^*(s,h),h)^{\top} \bV^*_{h+1} + 2\sqrt{\frac{\Var_{s'\sim \hat{\bp}_{t-1}(s,\pi^*(s,h),h)}\left(\bar{V}_{t,h+1}(s')\right)L}{N_{t-1}(s,\pi^*(s,h),h)}}\notag\\
         &+2\sqrt{\frac{\E_{s'\sim \hat{\bp}_{t-1}(s,\pi^*(s,h),h)}\left[\bar{V}_{t,h+1}(s')-V^*_{h+1}(s')\right]^2L}{N_{t-1}(s,\pi^*(s,h),h)}}+\frac{5HL}{N_{t-1}(s,\pi^*(s,h),h)}\\
         & \overset{(d)}{\ge} r(s,\pi^*(s,h),h)+\hat{\bp}_{t-1}(s,\pi^*(s,h),h)^{\top} \bV^*_{h+1} + 2\sqrt{\frac{\Var_{s'\sim \bp(s,\pi^*(s,h),h)}\left(V_{h+1}^*(s')\right)L}{N_{t-1}(s,\pi^*(s,h),h)}}+\frac{HL}{N_{t-1}(s,\pi^*(s,h),h)}\\
         &\overset{(e)}{\ge} r(s,\pi^*(s,h),h)+\bp(s,\pi^*(s,h),h)^{\top} \bV^*_{h+1}\\
         &=V_h^*(s),
     \end{align}
     where inequality (a) is due to taking the maximization over the actions in the optimistic MDP, inequality (b) is due to the inductive hypothesis $V^*_{h+1}(s)\le \bar{V}_{t,h+1}(s)$, inequality (c) is due to the inductive hypothesis $\ubar{V}_{t,h+1}(s)\le V^*_{h+1}(s)$, inequality (d) is due to \cref{apdx_lem:concen_var}, and inequality (e) is due to \cref{apdx_lem:concen_future}.

     Next we prove the pessimistic part. Let let $a=\pi_t(s,h)$. Similarly, if 
     \begin{align}
r(s,a,h)+\hat{\bp}_{t-1}(s,a,h)^{\top} \ubar{\bV}_{t,h+1} - \phi_t(s,a,h)\le 0, 
     \end{align}
     we are done. If the above inequality does not hold, we have
     \begin{align}
\ubar{V}_{t,h}(s)&=r(s,a,h)+\hat{\bp}_{t-1}(s,a,h)^{\top} \ubar{\bV}_{t,h+1} - \phi_t(s,a,h)\\
&\overset{(a)}\le r(s,a,h)+\hat{\bp}_{t-1}(s,a,h)^{\top} \bV_{h+1}^*-2\sqrt{\frac{\Var_{s'\sim \hat{\bp}_{t-1}(s,a,h)}\left(\bar{V}_{t,h+1}(s')\right)L}{N_{t-1}(s,a,h)}}\notag\\
&-2\sqrt{\frac{\E_{s'\sim \hat{\bp}_{t-1}(s,a,h)}\left[\bar{V}_{t,h+1}(s')-V^*_{h+1}(s')\right]^2L}{N_{t-1}(s,a,h)}}-\frac{5HL}{N_{t-1}(s,a,h)}\\
&\overset{(b)}\le r(s,a,h)+\hat{\bp}_{t-1}(s,a,h)^{\top} \bV^*_{h+1}-2\sqrt{\frac{\Var_{s'\sim \bp(s,a,h)}\left(V_{h+1}^*(s')\right)L}{N_{t-1}(s,a,h)}}-\frac{HL}{N_{t-1}(s,a,h)}\\
&\overset{(c)}{\le} r(s,a,h)+\bp(s,a,h)^{\top} \bV^*_{h+1}\\
&\le r(s,\pi^*(s,h),h)+\bp(s,\pi^*(s,h),h)^{\top} \bV^*_{h+1}\\
&=V_h^*(s),
     \end{align}
     where inequality (a) is due to the inductive hypothesis $V^*_{h+1}(s)\ge \ubar{V}_{t,h+1}(s)$, inequality (b) is due to \cref{apdx_lem:concen_var}, inequality (c) is due to \cref{apdx_lem:concen_future}.
\end{proof}

\begin{lemma}[Difference between optimism and pessimism]\label{apdx_lem:diff_onp}
If the concentration event $\cE$ holds, then we have 
\begin{align}
    \bar{V}_{t,h}(s)-\ubar{V}_{t,h}(s)\le \sum_{i=h}^H\E_{(s_i,a_i)\sim\pi_t}\left[\min\left\{\frac{20HL\sqrt{S}}{\sqrt{N_{t-1}(s_i,a_i,i)}},H\right\}\mid s_h=s,\pi_t\right]
\end{align}
    
\end{lemma}
\begin{proof}
    Let $a=\pi_t(s,h)$ be the action chosen by our algorithm at $t$-th episode,
    \begin{align}
        &\bar{V}_{t,h}(s)\le r(s,a,h)+\hat{\bp}_{t-1}(s,a,h)^{\top} \bar{\bV}_{t,h+1} + \phi_t(s,a,h)\\
        &\ubar{V}_{t,h}(s)\ge r(s,a,h)+\hat{\bp}_{t-1}(s,a,h)^{\top} \ubar{\bV}_{t,h+1} - \phi_t(s,a,h)
    \end{align}
    Then,
    \begin{align}
        &\bar{V}_{t,h}(s)-\ubar{V}_{t,h}(s)\\
        &\le \hat{\bp}_{t-1}(s,a,h)^{\top}\left(\bar{\bV}_{t,h+1}-\ubar{\bV}_{t,h+1}\right) + 2\phi_t(s,a,h)\\
        &=\bp(s,a,h)^{\top}\left(\bar{\bV}_{t,h+1}-\ubar{\bV}_{t,h+1}\right) +(\hat{\bp}_{t-1}(s,a,h)-\bp(s,a,h))^{\top}\left(\bar{\bV}_{t,h+1}-\ubar{\bV}_{t,h+1}\right)\notag\\
        &+4\sqrt{\frac{\Var_{s'\sim \hat{\bp}_{t-1}(s,a,h)}\left(\bar{V}_{t,h+1}(s')\right)L}{N_{t-1}(s,a,h)}}+4\sqrt{\frac{\E_{s'\sim \hat{\bp}_{t-1}(s,a,h)}\left[\bar{V}_{t,h+1}(s')-\ubar{V}_{t,h+1}(s')\right]^2L}{N_{t-1}(s,a,h)}}+\frac{10HL}{N_{t-1}(s,a,h)}\\
        &\le \bp(s,a,h)^{\top}\left(\bar{\bV}_{t,h+1}-\ubar{\bV}_{t,h+1}\right) +\norm{\hat{\bp}_{t-1}(s,a,h)-\bp(s,a,h)}_1\norm{\bar{\bV}_{t,h+1}-\ubar{\bV}_{t,h+1}}_{\infty}\notag\\
        &+4\sqrt{\frac{H^2L}{N_{t-1}(s,a,h)}}+4\sqrt{\frac{H^2L}{N_{t-1}(s,a,h)}}+\frac{10HL}{N_{t-1}(s,a,h)}\\
        &\overset{(a)}{\le} \bp(s,a,h)^{\top}\left(\bar{\bV}_{t,h+1}-\ubar{\bV}_{t,h+1}\right) + H\sqrt{\frac{2SL}{N_{t-1}(s,a,h)}} +18HL\sqrt{\frac{1}{N_{t-1}(s,a,h)}}\\
        &\le \bp(s,a,h)^{\top}\left(\bar{\bV}_{t,h+1}-\ubar{\bV}_{t,h+1}\right) + \frac{20HL\sqrt{S}}{\sqrt{N_{t-1}(s,a,h)}}\\
         &= \E_{s'\sim \bp(s,a,h)}\left[\bar{V}_{t,h+1}(s')-\ubar{V}_{t,h+1}(s')\right] + \frac{20HL\sqrt{S}}{\sqrt{N_{t-1}(s,a,h)}}\\
        &\overset{(b)}\le \sum_{i=h}^H\E_{(s_i,a_i)\sim\pi_t}\left[\min\left\{\frac{20HL\sqrt{S}}{\sqrt{N_{t-1}(s_i,a_i,i)}},H\right\}\mid s_h=s,\pi_t\right]
    \end{align}
    where inequality (a) is due to \cref{apdx_lem:concen_tran}, and inequality (b) is due to $a=\pi_t(s,h)$ and recursively apply the same operation on $\bar{V}_{t,h+1}(s')-\ubar{V}_{t,h+1}(s')$ till step $H$ when coupled with the fact that $\bar{V}_{t,h}(s)-\ubar{V}_{t,h}(s)\le H$.
\end{proof}

\begin{lemma}[Cumulative difference between optimism and pessimism]\label{apdx_lem:cum_diff_onp}
If the concentration event $\cE$ holds, then 
\begin{align}
    \sum_{h,s,a}q_{s,a,h}^{\bp, \pi_t}\bp(s,a,h)^{\top}(\bar{\bV}_{t,h+1}-\ubar{\bV}_{t,h+1})^2\le\sum_{h,s,a}q_{s,a,h}^{\bp, \pi_t}\frac{400H^4L^2S}{N_{t-1}(s,a,h)}
\end{align}
    
\end{lemma}
\begin{proof}
    For any $t\in[T], h\in [H], s\in \cS$, let $w_{t,h}(s)$ denote the probability that state $s$ is visited at step $h$ in episode $t$.
    We can bound 
    \begin{align}
        &\sum_{h,s,a}q_{s,a,h}^{\bp, \pi_t}\bp(s,a,h)^{\top}(\bar{\bV}_{t,h+1}-\ubar{\bV}_{t,h+1})^2\\
        &=\sum_{h,s,a}q_{s,a,h}^{\bp, \pi_t}\sum_{s'}p(s'|s,a,h)(\bar{V}_{t,h+1}(s')-\ubar{V}_{t,h+1}(s'))^2\\
        &=\sum_{h,s',s,a}q_{s,a,h}^{\bp, \pi_t}p(s'|s,a,h)(\bar{V}_{t,h+1}(s')-\ubar{V}_{t,h+1}(s'))^2\\
        &=\sum_{h,s'}w_{t,h+1}(s')(\bar{V}_{t,h+1}(s')-\ubar{V}_{t,h+1}(s'))^2\\
        &=\sum_{h=1}^H\E_{s_{h+1}\sim \pi_t}(\bar{V}_{t,h+1}(s_{h+1})-\ubar{V}_{t,h+1}(s_{h+1}))^2\\
        &\le \sum_{h=1}^H\E_{s_{h}\sim \pi_t}(\bar{V}_{t,h}(s_h)-\ubar{V}_{t,h}(s_h))^2\\
        &\overset{(a)}{\le} \sum_{h=1}^H\E_{s_{h}\sim \pi_t}\bigg(\sum_{i=h}^H\E_{(s_i,a_i)\sim\pi_t}\left[\frac{20HL\sqrt{S}}{\sqrt{N_{t-1}(s_i,a_i,h)}}\mid s_h,\pi_t\right]\bigg)^2\\
        &\overset{(b)}{\le} H\sum_{h=1}^H\E_{s_{h}\sim \pi_t}\sum_{i=h}^H\bigg(\E_{(s_i,a_i)\sim\pi_t}\left[\frac{20HL\sqrt{S}}{\sqrt{N_{t-1}(s_i,a_i,h)}}\mid s_h,\pi_t\right]\bigg)^2\\
        &\overset{(c)}{\le} H\sum_{h=1}^H\E_{s_{h}\sim \pi_t}\sum_{i=h}^H\E_{(s_i,a_i)\sim\pi_t}\left[\frac{400H^2L^2S}{N_{t-1}(s_i,a_i,h)}\mid s_h,\pi_t\right]\\
        &\le H\sum_{h=1}^H\sum_{i=h}^H\E_{(s_i,a_i)\sim\pi_t}\left[\frac{400H^2L^2S}{N_{t-1}(s_i,a_i,h)}\right]\\
        &\le H^2\sum_{h=1}^H\E_{(s_h,a_h)\sim\pi_t}\left[\frac{400H^2L^2S}{N_{t-1}(s_h,a_h,h)}\right]\\
        &=\sum_{h,s,a}q_{s,a,h}^{\bp, \pi_t}\frac{400H^4L^2S}{N_{t-1}(s,a,h)}
    \end{align}
    where inequality (a) is due to \cref{apdx_lem:diff_onp}, inequality (b) is due to Cauchy-Schwarz inequality, inequality (c) is due to Jensen's inequality.
\end{proof}

\subsection{Variance Inequalities}

\begin{lemma}[Cumulative difference of the variance]\label{apdx_lem:cum_diff_var}
    If the concentration event $\cE$ holds, then it holds for all $t\in[T]$ that
    \begin{align}
        \sum_{h,s,a}q_{s,a,h}^{\bp, \pi_t}\sqrt{\frac{\Var_{s'\sim \bp(s,a,h)}\left(V_{h+1}^*(s')\right)L}{N_{t-1}(s,a,h)}}-\sum_{h,s,a}q_{s,a,h}^{\bp, \pi_t}\sqrt{\frac{\Var_{s'\sim \bp(s,a,h)}\left(V_{h+1}^{\pi_t}(s')\right)L}{N_{t-1}(s,a,h)}}\le\sqrt{H^2L}\sqrt{\sum_{h,s,a}\frac{q_{s,a,h}^{\bp, \pi_t}\cdot\Delta_{\pi_t}}{N_{t-1}(s,a,h)}}
    \end{align}
\end{lemma}
\begin{proof}
For any $t\in[T], h\in [H], s\in \cS$, let $w_{t,h}(s)$ denote the probability that state $s$ is visited at step $h$ in episode $t$.
    \begin{align}
        &\sum_{h,s,a}q_{s,a,h}^{\bp, \pi_t}\sqrt{\frac{\Var_{s'\sim \bp(s,a,h)}\left(V_{h+1}^*(s')\right)L}{N_{t-1}(s,a,h)}}-\sum_{h,s,a}q_{s,a,h}^{\bp, \pi_t}\sqrt{\frac{\Var_{s'\sim \bp(s,a,h)}\left(V_{h+1}^{\pi_t}(s')\right)L}{N_{t-1}(s,a,h)}}\\
        &\overset{(a)}{\le}\sqrt{L}\sum_{h,s,a}q_{s,a,h}^{\bp, \pi_t}\sqrt{\frac{\E_{s'\sim \bp(s,a,h)}\left[\left(V_{h+1}^*(s')-V_{h+1}^{\pi_t}(s')\right)^2\right]}{N_{t-1}(s,a,h)}}\\
        &\le \sqrt{L}\sqrt{\sum_{h,s,a}\frac{q_{s,a,h}^{\bp, \pi_t}}{N_{t-1}(s,a,h)}}\sqrt{\sum_{h,s,a}q_{s,a,h}^{\bp, \pi_t}\E_{s'\sim \bp(s,a,h)}\left[\left(V_{h+1}^*(s')-V_{h+1}^{\pi_t}(s')\right)^2\right]}\\
        &=\sqrt{L}\sqrt{\sum_{h,s,a}\frac{q_{s,a,h}^{\bp, \pi_t}}{N_{t-1}(s,a,h)}}\sqrt{\sum_{h,s,a}q_{s,a,h}^{\bp, \pi_t}\sum_{s'}p(s'|s,a,h)\left(V_{h+1}^*(s')-V_{h+1}^{\pi_t}(s')\right)^2}\\
        &\le\sqrt{HL}\sqrt{\sum_{h,s,a}\frac{q_{s,a,h}^{\bp, \pi_t}}{N_{t-1}(s,a,h)}}\sqrt{\sum_{h,s,a}q_{s,a,h}^{\bp, \pi_t}\sum_{s'}p(s'|s,a,h)\left(V_{h+1}^*(s')-V_{h+1}^{\pi_t}(s')\right)}\\
        &=\sqrt{HL}\sqrt{\sum_{h,s,a}\frac{q_{s,a,h}^{\bp, \pi_t}}{N_{t-1}(s,a,h)}}\sqrt{\sum_{h,s'}w_{t,h+1}(s')\left(V_{h+1}^*(s')-V_{h+1}^{\pi_t}(s')\right)}\\
        &\overset{(b)}{\le} \sqrt{HL}\sqrt{\sum_{h,s,a}\frac{q_{s,a,h}^{\bp, \pi_t}}{N_{t-1}(s,a,h)}}\sqrt{\sum_{h}\left(V_{1}^*(s_1)-V_{1}^{\pi_t}(s_1)\right)}\\
        &=\sqrt{H^2L}\sqrt{\sum_{h,s,a}\frac{q_{s,a,h}^{\bp, \pi_t}}{N_{t-1}(s,a,h)}}\sqrt{\left(V_{1}^*(s_1)-V_{1}^{\pi_t}(s_1)\right)}\\
        &=\sqrt{H^2L}\sqrt{\sum_{h,s,a}\frac{q_{s,a,h}^{\bp, \pi_t}\cdot\Delta_{\pi_t}}{N_{t-1}(s,a,h)}}
    \end{align}
    where inequality (a) is due to Eq. (48)-(52) in \cite{zanette2019tighter} and inequality (b) is due to Lemma 17 in \cite{zanette2019tighter}
\end{proof}

\begin{lemma}[Law of total variance]\label{apdx_lem:ltv}
For any policy $\pi$, then it holds that
\begin{align}
\sum_{h,s,a}q^{\bp,\pi}_{s,a,h}\Var_{s'\sim\bp(s,a,h)}\left(V_{h+1}^{\pi}(s')\right)\le H^2
\end{align}
\end{lemma}
\begin{proof}
Let $E_{\pi}[\cdot|s_1]$ is taken over the trajectories following policy $\pi$ starting from state $s_1$.
    \begin{align}
    &\sum_{h,s,a}q^{\bp,\pi}_{s,a,h}\Var_{s'\sim\bp(s,a,h)}\left(V_{h+1}^{\pi}(s')\right)\\
    &=\sum_{h=1}^H\E_{\pi}\left[\Var_{\pi}\left(V_{h+1}^{\pi}(s_{h+1})|s_h\right)|s_1\right]\\
    &=\E_{\pi}\left[\sum_{h=1}^H\Var_{\pi}\left(V_{h+1}^{\pi}(s_{h+1})|s_h\right)|s_1\right]\\
    &\overset{(a)}{=}\E_{\pi}\left[\left(\sum_{h=1}^Hr(s_h,\pi(s_h,h),h)-V_1^{\pi}(s_1)\right)^2|s_1\right]\\
    &\le H^2,
    \end{align}
    where (a) is due to Lemma 15 in \cite{zanette2019tighter}.
\end{proof}

\subsection{Proof of Theorem \ref{thm:reg_RL_tight}}
Now we prove the regret upper bound for \cref{thm:reg_RL_tight}. The proof of \cref{thm:reg_RL_tight} is slightly different from that of \cref{thm:reg_CMAB-MT} but follows the similar analysis ideas/steps and we show the details as follows.

We suppose the concentration event $\cE$ holds. According to our joint oracle \cref{alg:v-opt}, we have $\bar{\bV}_{t,h}=\bV_{t,h}^{\tilde{\bp}_t, \pi_t}$, where $(\pi_t, \tilde{\bp}_t)=\argmax_{\pi\in \Pi, \tilde{\bp}\in \cC_t(\pi)}V_1^{\tilde{\bp},\pi}(s_1)$. Let $L=\logL$ and $\phi_{t}(s,a,h)=2\sqrt{\frac{\Var_{s'\sim \hat{\bp}_{t-1}(s,a,h)}\left(\bar{V}_{t,h+1}(s')\right)L}{N_{t-1}(s,a,h)}}+2\sqrt{\frac{\E_{s'\sim \hat{\bp}_{t-1}(s,a,h)}\left[\bar{V}_{t,h+1}(s')-\ubar{V}_{t,h+1}(s')\right]^2L}{N_{t-1}(s,a,h)}}+\frac{5HL}{N_{t-1}(s,a,h)}$.

\textbf{Step 1: Regret decomposition.}
We use a similar argument for the regret decomposition (Step 1) in \cref{apdx_sec:proof_of_CMAB-MT}.
\begin{align}
    \Delta_{\pi_t}&\overset{\df}{=} \left(V^{\bp,\pi^*}_1(s_{1})-V^{\bp,\pi_t}_1(s_{1})\right)\\
    & \overset{(a)}{\le} 
    \left(V_{t,1}^{\tilde{\bp}_{t},\pi_t}(s_{1})- V_1^{\bp,\pi_t}(s_{1})\right)\\
    & \overset{(b)}{\le} \sum_{h\in[H],s\in \cS, a\in \cA}q_{s,a,h}^{\bp, \pi_t}\abs{(\tilde{\bp}_t(s,a,h)-\bp(s,a,h))^{\top}\bV_{t,h+1}^{\tilde{\bp}_t,\pi_t}}\\
    &\overset{(c)}=\underbrace{\sum_{h,s,a}q_{s,a,h}^{\bp, \pi_t}\abs{\left(\tilde{\bp}_t(s,a,h)-\hat{\bp}_{t-1}(s,a,h)\right)^{\top}\bV_{t,h+1}^{\tilde{\bp}_t,\pi_t}}}_{\text{Term (\rom{1}): Optimistic Future Value Regret}}\\
    &+\underbrace{\sum_{h,s,a}q_{s,a,h}^{\bp, \pi_t}\abs{\left(\hat{\bp}_{t-1}(s,a,h)-\bp(s,a,h)\right)^{\top}\bV_{h+1}^{*}}}_{\text{Term (\rom{2}): Optimal Future Value Regret}}\\
    &+\underbrace{\sum_{h,s,a}q_{s,a,h}^{\bp, \pi_t}\abs{\left(\hat{\bp}_{t-1}(s,a,h)-\bp(s,a,h)\right)^{\top}\left(\bV_{t,h+1}^{\tilde{\bp}_t,\pi_t}-\bV_{h+1}^{*}\right)}}_{\text{Term (\rom{3}): Lower-Order Regret}}.
\end{align}
where inequality (a)-(c) following the same reasoning of (a)-(c) in the Step 1 of \cref{apdx_sec:proof_of_CMAB-MT}. In what follows identify the $F_{t,s,a,h},G_{t,s,a,h},I_{t,s,a,h},J_{t,s,a,h}$ and their upper bounds $\bar{F},\bar{G},\bar{I},\bar{J}$ as in \cref{thm:reg_CMAB-MT}.


\textbf{Step 2: Bound the optimistic future value regret - Term (\rom{1})}

First, we can identify $F_{t,s,a,h}=2\sqrt{{\Var_{s'\sim \hat{\bp}_{t-1}(s,a,h)}\left(\bar{V}_{t,h+1}(s')\right)L}}+2\sqrt{{\E_{s'\sim \hat{\bp}_{t-1}(s,a,h)}\left[\bar{V}_{t,h+1}(s')-\ubar{V}_{t,h+1}(s')\right]^2L}}$, $I_{t,s,a,h}={5HL}$ from \cref{apdx_eq:explain_F_I}. Slightly different from inequality (f) in Step 1 of \cref{apdx_sec:proof_of_CMAB-MT}, $\sum_{s,a,h}q_{s,a,h}^{\bp,\pi}F_{t,s,a,h}^2$ will produce some lower-order terms $F'=O(H^4SL^3\sum_{s,a,h}{\frac{q_{s,a,h}^{\bp,\pi_t}}{N_{t-1}(s,a,h)}})$ so that $\sqrt{F'}\sqrt{\sum_{s,a,h}{\frac{q_{s,a,h}^{\bp,\pi_t}}{N_{t-1}(s,a,h)}}}$ can be merged into $\bar{I}$. 
Concretely, we have
\begin{align}
    \text{Term (\rom{1})}&=\sum_{h,s,a}q_{s,a,h}^{\bp, \pi_t}\abs{\left(\tilde{\bp}_t(s,a,h)-\hat{\bp}_{t-1}(s,a,h)\right)^{\top}\bV_{t,h+1}^{\tilde{\bp}_t,\pi_t}}\\
    &\overset{(a)}{\le}\sum_{h,s,a}q_{s,a,h}^{\bp, \pi_t}\cdot\phi_t(s,a,h)\\
    &\overset{(b)}{\le}\sum_{h,s,a}q_{s,a,h}^{\bp, \pi_t}\bigg(2\sqrt{\frac{\Var_{s'\sim \bp(s,a,h)}\left(V_{h+1}^*(s')\right)L}{N_{t-1}(s,a,h)}}+4\sqrt{\frac{\E_{s'\sim \hat{\bp}_{t-1}(s,a,h)}\left[\bar{V}_{t,h+1}(s')-\ubar{V}_{t,h+1}(s')\right]^2L}{N_{t-1}(s,a,h)}}
     \notag\\
     &+\frac{9HL}{N_{t-1}(s,a,h)}\bigg)\label{apdx_eq:explain_F_I}\\
     &\overset{(c)}{\le}2\sum_{h,s,a}q_{s,a,h}^{\bp, \pi_t}\sqrt{\frac{\Var_{s'\sim \bp(s,a,h)}\left(V_{h+1}^{\pi_t}(s')\right)L}{N_{t-1}(s,a,h)}}+2\sqrt{H^2L}\sqrt{\sum_{h,s,a}\frac{q_{s,a,h}^{\bp, \pi_t}\cdot\Delta_{\pi_t}}{N_{t-1}(s,a,h)}} +9HL\sum_{h,s,a}\frac{q_{s,a,h}^{\bp, \pi_t}}{N_{t-1}(s,a,h)}\notag\\
     &+4\sum_{h,s,a}q_{s,a,h}^{\bp, \pi_t}\sqrt{\frac{\E_{s'\sim \hat{\bp}_{t-1}(s,a,h)}\left[\bar{V}_{t,h+1}(s')-\ubar{V}_{t,h+1}(s')\right]^2L}{N_{t-1}(s,a,h)}}\\
     &{\le}2\sqrt{L}\sqrt{\sum_{h,s,a}q^{\bp,\pi_t}_{s,a,h}\Var_{s'\sim\bp(s,a,h)}\left(V_{h+1}^{\pi_t}(s')\right)}\sqrt{\sum_{h,s,a}\frac{q_{s,a,h}^{\bp, \pi_t}}{N_{t-1}(s,a,h)}} +2\sqrt{H^2L}\sqrt{\sum_{h,s,a}\frac{q_{s,a,h}^{\bp, \pi_t}\cdot\Delta_{\pi_t}}{N_{t-1}(s,a,h)}}\notag\\
     &+9HL\sum_{h,s,a}\frac{q_{s,a,h}^{\bp, \pi_t}}{N_{t-1}(s,a,h)}
     +4\underbrace{\sum_{h,s,a}q_{s,a,h}^{\bp, \pi_t}\sqrt{\frac{\E_{s'\sim \hat{\bp}_{t-1}(s,a,h)}\left[\bar{V}_{t,h+1}(s')-\ubar{V}_{t,h+1}(s')\right]^2L}{N_{t-1}(s,a,h)}}}_{\text{Term (\rom{1}.a)}}\\
     &\overset{(d)}{\le}2\sqrt{H^2L}\sqrt{\sum_{h,s,a}\frac{q_{s,a,h}^{\bp, \pi_t}}{N_{t-1}(s,a,h)}} +2\sqrt{H^2L}\sqrt{\sum_{h,s,a}\frac{q_{s,a,h}^{\bp, \pi_t}\cdot\Delta_{\pi_t}}{N_{t-1}(s,a,h)}}\notag\\
     &+9HL\sum_{h,s,a}\frac{q_{s,a,h}^{\bp, \pi_t}}{N_{t-1}(s,a,h)}
     +100 \sqrt{H^4SL^3}\sum_{h,s,a}\frac{q_{s,a,h}^{\bp, \pi_t}}{N_{t-1}(s,a,h)}\label{apdx_eq:term_1}
\end{align}
where inequality (a) is due to the definition of confidence region function $\cC_t(\pi)$ in \cref{eq:RL_conf_set_tight}, inequality (b) is due to \cref{apdx_lem:concen_var}, inequality (c) is due to \cref{apdx_lem:cum_diff_var}, inequality (d) is due to \cref{apdx_lem:ltv} and the Term (\rom{1}.a) bounded as follows.
Before we prove the Term (\rom{1}.a), we can see from \cref{apdx_eq:term_1} (and compared with \cref{apdx_eq:CMAB-MT_analysis_f}) that we equivalently have $\bar{F}=4H^2L$ and the additionally produced second, third, and fourth term in \cref{apdx_eq:term_1} can be merged together as the $\bar{I}$ term.
For Term (\rom{1}.a) we have,
\begin{align}
     \text{Term (\rom{1.a})}&=\sum_{h,s,a}q_{s,a,h}^{\bp, \pi_t}\sqrt{\frac{\E_{s'\sim \hat{\bp}_{t-1}(s,a,h)}\left[\bar{V}_{t,h+1}(s')-\ubar{V}_{t,h+1}(s')\right]^2L}{N_{t-1}(s,a,h)}}\\
     &\le \sqrt{L}\sqrt{\sum_{h,s,a}\frac{q_{s,a,h}^{\bp, \pi_t}}{N_{t-1}(s,a,h)}}\sqrt{\sum_{h,s,a}q_{s,a,h}^{\bp, \pi_t}\hat{\bp}_{t-1}(s,a,h)^{\top}(\bar{\bV}_{t,h+1}-\ubar{\bV}_{t,h+1})^2}\\
     &\le\sqrt{L}\sqrt{\sum_{h,s,a}\frac{q_{s,a,h}^{\bp, \pi_t}}{N_{t-1}(s,a,h)}}\bigg(\sqrt{\sum_{h,s,a}q_{s,a,h}^{\bp, \pi_t}\bp(s,a,h)^{\top}(\bar{\bV}_{t,h+1}-\ubar{\bV}_{t,h+1})^2}\notag\\
     &+\sqrt{\sum_{h,s,a}q_{s,a,h}^{\bp, \pi_t}\abs{\bp(s,a,h)-\hat{\bp}_{t-1}(s,a,h)}^{\top}(\bar{\bV}_{t,h+1}-\ubar{\bV}_{t,h+1})^2}\bigg)\\
     &\le\sqrt{L}\sqrt{\sum_{h,s,a}\frac{q_{s,a,h}^{\bp, \pi_t}}{N_{t-1}(s,a,h)}}\bigg(\sqrt{\sum_{h,s,a}q_{s,a,h}^{\bp, \pi_t}\bp(s,a,h)^{\top}(\bar{\bV}_{t,h+1}-\ubar{\bV}_{t,h+1})^2}\notag\\
     &+\sqrt{H}\sqrt{\sum_{h,s,a}q_{s,a,h}^{\bp, \pi_t}\abs{\bp(s,a,h)-\hat{\bp}_{t-1}(s,a,h)}^{\top}(\bar{\bV}_{t,h+1}-\ubar{\bV}_{t,h+1})}\bigg)\\
     &\overset{(a)}\le \sqrt{L}\sqrt{\sum_{h,s,a}\frac{q_{s,a,h}^{\bp, \pi_t}}{N_{t-1}(s,a,h)}}\bigg(\sqrt{\sum_{h,s,a}q_{s,a,h}^{\bp, \pi_t}\frac{400H^4L^2S}{N_{t-1}(s,a,h)}}+\sqrt{H}\sqrt{21\sqrt{H^4S^2L^3}\sum_{h,s,a}\frac{q_{s,a,h}^{\bp, \pi_t}}{N_{t-1}(s,a,h)}}\bigg)\\
     &\le20\sqrt{H^4SL^3}\sum_{h,s,a}\frac{q_{s,a,h}^{\bp, \pi_t}}{N_{t-1}(s,a,h)}+5\sqrt{H^3SL^{2.5}}\sum_{h,s,a}\frac{q_{s,a,h}^{\bp, \pi_t}}{N_{t-1}(s,a,h)}\\
     &\le 25 \sqrt{H^4SL^3}\sum_{h,s,a}\frac{q_{s,a,h}^{\bp, \pi_t}}{N_{t-1}(s,a,h)},
\end{align}
where inequality (a) is due to \cref{apdx_lem:cum_diff_onp} and the Term (\rom{3}) bounded by \cref{apdx_eq:lower_order}. Intuitively, term 

\textbf{Step 3: Bound the optimal future value regret - Term (\rom{2})}

Step 3 can be proved using Step 2 above since $\left[\left(\hat{\bp}_{t-1}(s,a,h)-\bp(s,a,h)\right)^{\top}\bV_{h+1}^{*}\right]\le \phi_t(s,a,h)$ due to \cref{apdx_lem:concen_known_future}. But we can have a tighter bound as follows:

\begin{align}
    \text{Term (\rom{2})}&=\sum_{h,s,a}q_{s,a,h}^{\bp, \pi_t}\abs{\left(\hat{\bp}_{t-1}(s,a,h)-\bp(s,a,h)\right)^{\top}\bV_{h+1}^{*}}\\
    &\overset{(a)}{\le} 2\sum_{h,s,a}q_{s,a,h}^{\bp, \pi_t}\sqrt{\frac{\Var_{s'\sim \bp(s,a,h)}\left(V_{h+1}^*(s')\right)L}{N_{t-1}(s,a,h)}}+\sum_{h,s,a}q_{s,a,h}^{\bp, \pi_t}\frac{HL}{N_{t-1}(s,a,h)}\\
    &\overset{(b)}{\le}2\sum_{h,s,a}q_{s,a,h}^{\bp, \pi_t}\sqrt{\frac{\Var_{s'\sim \bp(s,a,h)}\left(V_{h+1}^{\pi_t}(s')\right)L}{N_{t-1}(s,a,h)}}+2\sqrt{H^2L}\sqrt{\sum_{h,s,a}\frac{q_{s,a,h}^{\bp, \pi_t}\cdot\Delta_{\pi_t}}{N_{t-1}(s,a,h)}} +HL\sum_{h,s,a}\frac{q_{s,a,h}^{\bp, \pi_t}}{N_{t-1}(s,a,h)}\\
    &{\le}2\sqrt{L}\sqrt{\sum_{h,s,a}q^{\bp,\pi_t}_{s,a,h}\Var_{s'\sim\bp(s,a,h)}\left(V_{h+1}^{\pi_t}(s')\right)}\sqrt{\sum_{h,s,a}\frac{q_{s,a,h}^{\bp, \pi_t}}{N_{t-1}(s,a,h)}} +2\sqrt{H^2L}\sqrt{\sum_{h,s,a}\frac{q_{s,a,h}^{\bp, \pi_t}\cdot\Delta_{\pi_t}}{N_{t-1}(s,a,h)}}\notag\\
    &+HL\sum_{h,s,a}\frac{q_{s,a,h}^{\bp, \pi_t}}{N_{t-1}(s,a,h)}\\
    &\overset{(c)}{\le}2\sqrt{H^2L}\sqrt{\sum_{h,s,a}\frac{q_{s,a,h}^{\bp, \pi_t}}{N_{t-1}(s,a,h)}}+2\sqrt{H^2L}\sqrt{\sum_{h,s,a}\frac{q_{s,a,h}^{\bp, \pi_t}\cdot\Delta_{\pi_t}}{N_{t-1}(s,a,h)}}+HL\sum_{h,s,a}\frac{q_{s,a,h}^{\bp, \pi_t}}{N_{t-1}(s,a,h)}\label{apdx_eq:term_2}
\end{align}
where inequality (a) is due to \cref{apdx_lem:concen_future}, inequality (b) is due to the \cref{apdx_lem:cum_diff_var}, and inequality (c) is due to \cref{apdx_lem:ltv}.

\textbf{Step 4: Bound the lower-order regret - Term (\rom{3})}

For Term (\rom{3}), we can identify $G_{t,s,a,h}=\sum_{s'} \bigg(\sqrt{{p(s'|s,a,h)(1-p(s'|s,a,h))L}}\bigg)\left(\bar{V}_{t,h+1}(s')-V_{h+1}^{*}(s')\right)$, $J_{t,s,a,h}=HSL$ as in \cref{apdx_eq:explain_G_J}. Then we can show that $\sum_{s,a,h}q_{s,a,h}^{\bp,\pi_t}G_{t,s,a,h}^2$ will produce lower order terms  $G'=O(H^4S^2L^3\sum_{s,a,h}{\frac{q_{s,a,h}^{\bp,\pi_t}}{N_{t-1}(s,a,h)}})$ so that $\sqrt{G'}\sqrt{\sum_{s,a,h}{\frac{q_{s,a,h}^{\bp,\pi_t}}{N_{t-1}(s,a,h)}}}$ can be merged into $\bar{J}$. Therefore it is equivalent to have $\bar{G}=0,\bar{J}=21\sqrt{H^4S^2L^3}$ as follows. Concretely, we have

\begin{align}
    \text{Term (\rom{3})}&=\sum_{h,s,a}q_{s,a,h}^{\bp, \pi_t}\abs{\left(\hat{\bp}_{t-1}(s,a,h)-\bp(s,a,h)\right)^{\top}\left(\bV_{t,h+1}^{\tilde{\bp}_t,\pi_t}-\bV_{h+1}^{*}\right)}\\
    &= \sum_{h,s,a}q_{s,a,h}^{\bp, \pi_t}\abs{\left(\hat{\bp}_{t-1}(s,a,h)-\bp(s,a,h)\right)^{\top}\left(\bar{\bV}_{t,h+1}-\bV_{h+1}^{*}\right)}\\
    &\le\sum_{h,s,a}q_{s,a,h}^{\bp, \pi_t}\sum_{s'}\abs{\hat{p}_{t-1}(s'|s,a,h)-p(s'|s,a,h)}\left(\bar{V}_{t,h+1}(s')-V_{h+1}^{*}(s')\right)\\
    &\overset{(a)}{\le}\sum_{h,s,a}q_{s,a,h}^{\bp, \pi_t}\sum_{s'} \bigg(\sqrt{\frac{p(s'|s,a,h)(1-p(s'|s,a,h))L}{N_{t-1}(s,a,h)}}+\frac{L}{N_{t-1}(s,a,h)}\bigg)\left(\bar{V}_{t,h+1}(s')-V_{h+1}^{*}(s')\right)\\
    &\overset{(b)}{\le} \sum_{h,s,a}q_{s,a,h}^{\bp, \pi_t}\sum_{s'} \sqrt{\frac{p(s'|s,a,h)(1-p(s'|s,a,h))L}{N_{t-1}(s,a,h)}}\left(\bar{V}_{t,h+1}(s')-V_{h+1}^{*}(s')\right) +\sum_{h,s,a}q_{s,a,h}^{\bp, \pi_t}\frac{HSL}{N_{t-1}(s,a,h)}\label{apdx_eq:explain_G_J}\\
    &\overset{(c)}{\le} \sqrt{SL}\sum_{h,s,a}q_{s,a,h}^{\bp, \pi_t} \sqrt{\frac{\sum_{s'}p(s'|s,a,h)\left(\bar{V}_{t,h+1}(s')-V_{h+1}^{*}(s')\right)^2}{N_{t-1}(s,a,h)}} +\sum_{h,s,a}q_{s,a,h}^{\bp, \pi_t}\frac{HSL}{N_{t-1}(s,a,h)}\\
    &\le \sqrt{SL}\sqrt{\sum_{h,s,a}\frac{q_{s,a,h}^{\bp, \pi_t}}{N_{t-1}(s,a,h)}}\sqrt{\sum_{h,s,a}q_{s,a,h}^{\bp, \pi_t} \bp(s,a,h)^{\top}\left(\bar{\bV}_{t,h+1}-\bV_{h+1}^{*}\right)^2} +\sum_{h,s,a}q_{s,a,h}^{\bp, \pi_t}\frac{HSL}{N_{t-1}(s,a,h)}\\
    &\le \sqrt{SL}\sqrt{\sum_{h,s,a}\frac{q_{s,a,h}^{\bp, \pi_t}}{N_{t-1}(s,a,h)}}\sqrt{\sum_{h,s,a}q_{s,a,h}^{\bp, \pi_t} \bp(s,a,h)^{\top}\left(\bar{\bV}_{t,h+1}-\ubar{\bV}_{t,h+1}\right)^2} +\sum_{h,s,a}q_{s,a,h}^{\bp, \pi_t}\frac{HSL}{N_{t-1}(s,a,h)}\\
    &\overset{(d)}{\le} \sqrt{SL}\sqrt{\sum_{h,s,a}\frac{q_{s,a,h}^{\bp, \pi_t}}{N_{t-1}(s,a,h)}}\sqrt{\sum_{h,s,a}q_{s,a,h}^{\bp, \pi_t}\frac{400H^4L^2S}{N_{t-1}(s,a,h)}} +\sum_{h,s,a}q_{s,a,h}^{\bp, \pi_t}\frac{HSL}{N_{t-1}(s,a,h)} \\
    &\le 21\sqrt{H^4S^2L^3}\sum_{h,s,a}\frac{q_{s,a,h}^{\bp, \pi_t}}{N_{t-1}(s,a,h)}\label{apdx_eq:lower_order}
\end{align}
where inequality (a) is due to \cref{apdx_lem:concen_tran}, inequality (b) is due to $\bar{V}_{t,h+1}(s')-V_{h+1}^{*}(s')\le H$, inequality (c) is due to Cauchy-Schwarz inequality, and inequality (d) is due to \cref{apdx_lem:cum_diff_onp}.

\textbf{Step 5: Putting all together and using CMAB-MT techniques in \cref{apdx_sec:CMAB-MT}}

Using \cref{apdx_eq:term_1}, \cref{apdx_eq:term_2} and \cref{apdx_eq:lower_order}, we have
\begin{align}
    \Delta_{\pi_t}&\le2\sqrt{H^2L}\sqrt{\sum_{h,s,a}\frac{q_{s,a,h}^{\bp, \pi_t}}{N_{t-1}(s,a,h)}} +2\sqrt{H^2L}\sqrt{\sum_{h,s,a}\frac{q_{s,a,h}^{\bp, \pi_t}\cdot\Delta_{\pi_t}}{N_{t-1}(s,a,h)}}\notag+9HL\sum_{h,s,a}\frac{q_{s,a,h}^{\bp, \pi_t}}{N_{t-1}(s,a,h)}
     \notag\\
    &+100 \sqrt{H^4SL^3}\sum_{h,s,a}\frac{q_{s,a,h}^{\bp, \pi_t}}{N_{t-1}(s,a,h)}+2\sqrt{H^2L}\sqrt{\sum_{h,s,a}\frac{q_{s,a,h}^{\bp, \pi_t}}{N_{t-1}(s,a,h)}}+2\sqrt{H^2L}\sqrt{\sum_{h,s,a}\frac{q_{s,a,h}^{\bp, \pi_t}\cdot\Delta_{\pi_t}}{N_{t-1}(s,a,h)}}\notag\\
    &+HL\sum_{h,s,a}\frac{q_{s,a,h}^{\bp, \pi_t}}{N_{t-1}(s,a,h)}+21\sqrt{H^4S^2L^3}\sum_{h,s,a}\frac{q_{s,a,h}^{\bp, \pi_t}}{N_{t-1}(s,a,h)}\\
    &\le 4\sqrt{H^2L}\sqrt{\sum_{h,s,a}\frac{q_{s,a,h}^{\bp, \pi_t}}{N_{t-1}(s,a,h)}} +4\sqrt{H^2L}\sqrt{\sum_{h,s,a}\frac{q_{s,a,h}^{\bp, \pi_t}\cdot\Delta_{\pi_t}}{N_{t-1}(s,a,h)}}
     +131 \sqrt{H^4S^2L^3}\sum_{h,s,a}\frac{q_{s,a,h}^{\bp, \pi_t}}{N_{t-1}(s,a,h)}\label{apdx_eq:RL_optimal_reg_decomp}
\end{align}

Let $c_1=4\times4\sqrt{H^2L}, c_2=4\times 4\sqrt{H^2L}, c_3=4\times 131 \sqrt{H^4S^2L^3}, c_4=4 \times 2H$. We define the four decomposed events as follows.
\begin{align}
    \cE'_{t,1}&=\left\{\Delta_{\pi_t}\le c_1\sqrt{\sum_{s,a,h:N_{t-1}(s,a,h)>0}\frac{q_{s,a,h}^{\bp,\pi_t}}{N_{t-1}(s,a,h)}}\right\},
    \cE'_{t,2}=\left\{\Delta_{\pi_t}\le c_2\sqrt{\sum_{s,a,h:N_{t-1}(s,a,h)>0}\frac{q_{s,a,h}^{\bp,\pi_t}\Delta_{\pi_t}}{N_{t-1}(s,a,h)}}\right\},\\
    \cE'_{t,3}&=\left\{\Delta_{\pi_t}\le c_3\sum_{s,a,h:N_{t-1}(s,a,h)>0}\frac{q_{s,a,h}^{\bp,\pi_t}}{N_{t-1}(s,a,h)}\right\},
     \cE'_{t,3}=\left\{\Delta_{\pi_t}\le c_4\sum_{s,a,h:N_{t-1}(s,a,h)=0}{q_{s,a,h}^{\bp,\pi_t}}\cdot2H\right\}.
\end{align}

By \cref{apdx_lem:decompose_filter_reg} and \cref{apdx_eq:RL_optimal_reg_decomp}, we have
\begin{align}
    \reg(T,\cE)\le \sum_{i=1}^4\reg(T,\cE'_{t,i})
\end{align}

Regarding $\reg(T,\cE'_{t,i})$ for $i=1,3,4$, we can apply similar analysis to that of steps 2,3,4 in \cref{apdx_sec:proof_of_CMAB-MT} respectively.

For $\reg(T,\cE'_{t,2})$, if $\cE'_{t,2}$ holds, then we have
\begin{align}
    \Delta_{\pi_t}\le c_2^2\sum_{s,a,h:N_{t-1}(s,a,h)>0}\frac{q_{s,a,h}^{\bp,\pi_t}}{N_{t-1}(s,a,h)}
\end{align}
which can be bounded exactly the same way as $\reg(T,\cE'_{t,3})$.

Using a similar analysis to \cref{apdx_sec:proof_of_CMAB-MT}, we have
\begin{align}\label{apdx_eq:gap_depnedent_reg}
    \reg(T,\cE_t)&\le \sum_{i \in [m]} \frac{2c_1^2}{\Delta_{i}^{\min} } (3+ \log K) + \sum_{i \in [m]} 2(c_2^2+c_3)  \left(1+\log\left (\frac{2(c_2^2+c_3)K}{\Delta_{i}^{\min}}\right)\right) + c_4m\\
    &=\sum_{s,a,h} \frac{512H^2L}{\Delta_{s,a,h}^{\min} } (3+ \log H) + \sum_{s,a,h} 1560H^2SL^{1.5}\left(1+\log\left (\frac{1560H^3SL^{1.5}}{\Delta_{s,a,h}^{\min}}\right)\right) + 8SAH^2\\
    &=O\left(\sum_{s,a,h} \frac{H^2L}{\Delta_{s,a,h}^{\min} } + \sum_{s,a,h}H^2SL^{1.5}\log\left (\frac{1}{\Delta_{s,a,h}^{\min}}\right)\right)
\end{align}

Similar to the analysis of \cref{apdx_sec:proof_of_CMAB-MT}, the gap-independent regret bound is $\tilde{O}(\sqrt{H^3SAT}+H^3S^{2}A)$ when considering the inhomogeneous episodic RL setting.

\subsection{Discussion about Gap-Dependent Regret Bound}\label{apdx_sec:proof_gap_dependent_bound}
In this section, we discuss the tightness of our gap-dependent bound in \cref{apdx_eq:gap_depnedent_reg}. Since we use a different definition of the gap, it is not directly comparable to the existing works such as \cite{simchowitz2019non}. Here we specify the value of $\Delta_{s,a,h}^{\min}=\min_{\pi\in \Pi: q^{\bp,\pi}_{s,a,h}>0, V_1^*(s_1)-V_1^{\pi}(s_1)>0} (V_1^*(s_1)-V_1^{\pi}(s_1))$, where we will omit the underlying transition probabilities $\bp$ for $V$ and $Q$ functions. Similar to \citet{simchowitz2019non}, we first divide $(s,a,h)$ into two parts: $\cZ_{sub}=\{(s,a,h): \pi^*(s,h)\neq a\}$ and $\cZ_{opt}=\cS\times\cA\times [H]-\cZ_{sub}$.
We use $\text{gap}(s,a,h)=V^{*}_h(s)-Q^{*}_h(s,a)$ to denote the state-dependent suboptimality gap, and  $\text{gap}_{\min}=\min_{s,a,h}\{\text{gap}(s,a,h):\text{gap}(s,a,h)>0\}$ the minimum gap. Let $q^*=\min_{\pi, (s,a,h)}\{q_{s,a,h}^{\bp,\pi}:q_{s,a,h}^{\bp,\pi}>0\}$ be the minimum occupancy measure for any policy $\pi$ and state-action-step pair $(s,a,h)$.

We use the following performance difference lemma for episodic MDP as follows, which is slightly different from Lemma 1.16 for infinite horizon discounted MDP in \citet{agarwal2019reinforcement}:
\begin{lemma}[Performance difference lemma for episodic MDP]
    For any MDP with transition kernel $\bp$ and for any two policies $\pi$ and $\pi'$, the difference of their value function starting from the initial state $s_1$ can be bounded by
\begin{align}
    V_1^{\pi}(s_1)-V_1^{\pi'}(s_1)=\sum_{s,a,h}q_{s,a,h}^{\bp,\pi'}\left[V^{\pi}_h(s)-Q^{\pi}_h(s,a)\right]
\end{align}
\end{lemma}

\begin{proof}
Let $q_{s,h}^{\bp,\pi}$ be the probability of visiting state $s$ at step $h$ following policy $\pi$.
    \begin{align}
        V_1^{\pi}(s_1)-V_1^{\pi'}(s_1)&=V_1^{\pi}(s_1)-\sum_{s,a,h}q^{\bp,\pi'}_{s,a,h} r(s,a,h)\\
        &\overset{(a)}{=}V_1^{\pi}(s_1)-\sum_{s,a,h}q^{\bp,\pi'}_{s,a,h} \left[Q_h^{\pi}(s,a)-\sum_{s'}p_h(s'\mid s,a,h)V_{h+1}^{\pi}(s')\right]\\
        &=V_1^{\pi}(s_1)+\sum_{s,a,h}\sum_{s'}q^{\bp,\pi'}_{s,a,h}p_h(s'\mid s,a,h)V_{h+1}^{\pi}(s')
        -\sum_{s,a,h}q^{\bp,\pi'}_{s,a,h}Q_h^{\pi}(s,a)\\
        &\overset{(b)}{=}V_1^{\pi}(s_1)+\sum_{s',h}q^{\bp,\pi'}_{s',h+1}V_{h+1}^{\pi}(s')-\sum_{s,a,h}q^{\bp,\pi'}_{s,a,h}Q_h^{\pi}(s,a)\\
        &\overset{(c)}{=}\sum_{s',h}q^{\bp,\pi'}_{s',h}V_{h}^{\pi}(s')-\sum_{s,a,h}q^{\bp,\pi'}_{s,a,h}Q_h^{\pi}(s,a)\\
        &=\sum_{s,a,h}q_{s,a,h}^{\bp,\pi'}\left[V^{\pi}_h(s)-Q^{\pi}_h(s,a)\right]
    \end{align}
\end{proof}
where equality (a) is due to the Bellman equation $Q_h^{\pi}(s,a)=r(s,a,h)+\sum_{s'}p(s'\mid s,a,h)V_{h+1}^{\pi}(s')$, equality (b) is due to $\sum_{s,a}q^{\bp,\pi}_{s,a,h}p(s'\mid s,a,h)=q^{\bp,\pi}_{s',h+1}$, and equality (c) is due to $q^{\bp,\pi}_{s_1,1}=1$ and $q^{\bp,\pi}_{s,1}=0$ for $s\neq s_1$.

Thus $V_1^{*}(s_1)-V_1^{\pi}(s_1)=\sum_{s,a,h}q_{s,a,h}^{\bp,\pi}\text{gap}(s,a,h)$, and for $(s,a,h)\in \cZ_{\text{sub}}$, we have $\Delta_{s,a,h}^{\min}=\min_{\pi\in \Pi: q^{\bp,\pi}_{s,a,h}>0, V_1^*(s_1)-V_1^{\pi}(s_1)>0} (V_1^*(s_1)-V_1^{\pi}(s_1))\ge q^* \cdot\text{gap}(s,a,h)$. For $(s,a,h)\in \cZ_{\text{opt}}$, we have $\Delta_{s,a,h}^{\min}\ge \text{gap}_{\min}$ since in the worst case, $\pi$ allocates all the triggering probability $q_{s,a,h}^{\bp,\pi}$ to  the $(s,a,h)$ that attains $\text{gap}_{\min}$. Now based on \cref{apdx_eq:gap_depnedent_reg} and the above reasoning, we have
\begin{align}
    \reg(T)&\le {O}\Bigg(\sum_{(s,a,h)\in \cZ_{\text{sub}}} \frac{H^2}{q^* \cdot\text{gap}(s,a,h) }\log (SAHT)+ \frac{H^2|\cZ_{\text{opt}}|}{\text{gap}_{\min}}\log (SAHT) \notag\\
    &+H^3S^2A\log^{1.5} (SAHT)\log\left (\frac{1}{\text{gap}_{\min}}\log^{1.5} (SAHT)\right)\Bigg)
\end{align}
which matches the regret bound of \cite{simchowitz2019non} up to a factor of $1/q^*$.
 
\section{Analysis for PMC-GD in Section~\ref{sec:CDP}}\label{apdx_sec:beyond_RL}
\subsection{Proof of Lemma~\ref{lem:OCD_smooth}}
\begin{align}
    \abs{r(\pi;\tilde{\bp})-r(\pi;\bp)}&=\abs{\sum_{v \in V}\left[\left(1-\prod_{u\in \pi}\left(1-\tilde{p}(u,v)\right)\right)-\left(1-\prod_{u\in \pi}\left(1-p(u,v)\right)\right) \right] }\\
    &=\abs{\sum_{v \in V}\left[\prod_{u\in \pi}\left(1-p(u,v)\right)-\prod_{u\in \pi}\left(1-\tilde{p}(u,v)\right) \right] }\\
    &\le \sum_{v\in V} \abs{\prod_{u\in \pi}\left(1-p(u,v)\right)-\prod_{u\in \pi}\left(1-\tilde{p}(u,v)\right)  }\\
    &\overset{(a)}{\le} \sum_{u\in\pi}\norm{\tilde{\bp}(u,\cdot)-\bp(u,\cdot)}_1
\end{align}
where inequality (a) is due to the fact that let $(a_1, ..., a_{|\pi|})\overset{\df}{=}(1-p(u,v))_{u\in \pi}$, $(b_1, ..., b_{|\pi|})\overset{\df}{=}(1-\tilde{p}(u,v))_{u\in \pi}$, $\abs{\prod_{i=1}^{|\pi|}a_i-\prod_{i=1}^{|\pi|}b_i}=\abs{\sum_{i=1}^{|\pi|}\prod_{j=1}^{i-1}a_j\cdot(a_i-b_i)\cdot\prod_{k=i+1}^{|\pi|}b_k}\le \sum_{i=1}^{|\pi|}|a_i-b_i|$.

For the pseudo-reward function $\bar{r}_t(\pi;\tilde{\bp})=r(\pi;\hat{\bp}_{t-1})+\sum_{u\in \pi}\norm{\tilde{\bp}(u,\cdot)-\hat{\bp}_{t-1}(u,\cdot)}_1$, we also have for all $\pi,\bp,\tilde{\bp}$, it holds
\begin{align}
\abs{\bar{r}_t(\pi;\tilde{\bp})-\bar{r}_t(\pi;\bp)}&=\abs{\sum_{u\in\pi}\norm{\tilde{\bp}(u,\cdot)-\hat{\bp}_{t-1}(u,\cdot)}_1 -\norm{\bp(u,\cdot)-\hat{\bp}_{t-1}(u,\cdot)}_1}\\
    &\le \sum_{u\in\pi}\abs{\norm{\tilde{\bp}(u,\cdot)-\hat{\bp}_{t-1}(u,\cdot)}_1 -\norm{\bp(u,\cdot)-\hat{\bp}_{t-1}(u,\cdot)}_1}\\
    &\overset{(a)}{\le} \sum_{u\in\pi}\norm{\tilde{\bp}(u,\cdot)-\bp(u,\cdot)}_1 \label{apdx_eq:pseudo_smooth}
\end{align}
where inequality (a) is due to $-\norm{\bx-\by}_1\le \norm{\bx}_1-\norm{\by}_1\le \norm{\bx-\by}_1$ by triangle inequality.

\subsection{Proof of Theorem~\ref{thm:reg_PMC_GD}}
We define the concentration event as:
\begin{align}
    \cE\overset{\df}{=} \Big[&\norm{\hat{\bp}_{t-1}(u,\cdot)-\bp(u,\cdot)}_1\le \sqrt{\frac{2|V|\log\left(\frac{|U||V|T}{\delta'}\right)}{N_{t-1,u}}},\text{ for any }u\in U, t\in [T]\Big]\label{apdx_eq:event_PMC_GD}
\end{align}

    Suppose the concentration event $\cE$ holds with probability $\delta'=1/(2T)$ as in \cref{apdx_lem:concen_tran}.
    Let $\alpha=1-1/e$ and $L=\log (|U||V|T)$. Also we can initialize each counter by $N_{t_0,u}=1$ using $t_0=|U|$ rounds which pays an extra $O(k|U|)$ regret. Now we have
    \begin{align}
    &\Delta_{\pi_t}=\alpha \cdot r(\pi^*;\bp)- r(\pi_t;\bp)\\
     &\overset{(a)}\le\alpha \cdot \bar{r}_t(\pi^*;\bp)- r(\pi_t;\bp)\\
    &\overset{(b)}{\le}\alpha\bar{r}_t(\pi^*;\tilde{\bp}_t)-r(\pi_t;\bp)\\
    &\overset{(c)}{\le}\bar{r}_t(\pi_t;\tilde{\bp}_t)-r(\pi_t;\bp)\\
    &{=} \bar{r}_t(\pi_t;\tilde{\bp}_t)-\bar{r}_t(\pi_t;\bp)+\bar{r}_t(\pi_t;\bp)-r(\pi_t;\bp)\\
     &\overset{(d)}\le \sum_{u\in\pi_t}\norm{\tilde{\bp}_t(u,\cdot)-\bp(u,\cdot)}_1+\bar{r}_t(\pi_t;\bp)-r(\pi_t;\bp)\\
    &\overset{(e)}\le 4\sum_{u\in \pi_t}\sqrt{\frac{|V|L}{N_{t-1,u}}}+\underbrace{\bar{r}_t(\pi_t;\bp)-r(\pi_t;\bp)}_{\text{Additional Term}}\\
    &\overset{(f)}{\le} 8\sum_{u\in \pi_t}\sqrt{\frac{|V|L}{N_{t-1,u}}},
\end{align}
where inequality (a) is due to $\bar{r}_t(\pi;\bp)\ge r(\pi;\bp)$ for any $\pi,\bp$ by \cref{lem:OCD_smooth},
inequality (b) is due to the definition of $\tilde{\bp}_t$ in \cref{alg:oracle_PMC_GD}, inequality (c) is due to $\pi_t$ is a $(1-1/e,1)$-approximate solution to the problem $\argmax_{|\pi|\le k}\bar{r}_t(\pi;\tilde{\bp}_t)$, 
inequality (d) is by \cref{apdx_eq:pseudo_smooth}, inequality (e) is due to \cref{apdx_eq:event_PMC_GD} and \cref{eq:conf_set_PMC_GD}, and inequality (f) is due to the following inequality to deal with additional regret term brought by pseudo-reward $\bar{r}_t(\pi;\bp)$.
\begin{align}
    \text{Additional Term}&=\bar{r}_t(\pi_t;\bp)-r(\pi_t;\bp)\\
    &=\sum_{u\in \pi_t}\norm{\bp(u,\cdot)-\hat{\bp}_{t-1}(u,\cdot)}_1+r(\pi_t;\hat{\bp}_{t-1})-r(\pi_t;\bp)\\
    &\overset{(a)}{\le} 2\sum_{u\in \pi_t}\norm{\bp(u,\cdot)-\hat{\bp}_{t-1}(u,\cdot)}_1\\
    &\overset{(b)}{\le}4\sum_{u\in \pi_t}\sqrt{\frac{|V|L}{N_{t-1,u}}}\\
\end{align}
where inequality (a) is due to \cref{lem:OCD_smooth}, and inequality (b) is due to event $\cE$.

Compared with \cref{apdx_eq:CMAB-MT_analysis_d}, it is equivalent to have $F_{t,u}=8\sqrt{|V|L}, G_{t,u}=I_{t,u}=J_{t,u}=0$ and
following step 2 in \cref{apdx_sec:proof_of_CMAB-MT} where $\bar{F}=64k|V|L, \bar{G}=\bar{I}=\bar{J}=0$, we have gap-dependent regret
\begin{align}
Reg(T)=O\left(\sum_{u\in U}\frac{k|V|\log (|U||V|T)}{\Delta_{u}^{\min}}\right)
\end{align}
and gap-independent regret
\begin{align}
Reg(T)=O\left(\sqrt{k|V||U|T\log (|U||V|T)}\right).
\end{align}

\end{document}